\documentclass[graybox]{svmult}

\usepackage{type1cm}        
\usepackage{makeidx}         
\usepackage{graphicx}        
\usepackage{multicol}        
\usepackage[bottom]{footmisc}

\usepackage{newtxtext}       %
\usepackage{newtxmath}       
\usepackage{xcolor}
\usepackage{hyperref}
\usepackage{url}
\usepackage{subcaption}
\usepackage{float}
\usepackage{booktabs}

\makeindex             

\newcommand{\R}{\mathbb{R}}

\newcommand{\norm}[1]{\left\lVert#1\right\rVert}

\title*{On the well-posedness of uncalibrated photometric stereo under general lighting}
\author{Mohammed Brahimi, Yvain Qu\'eau, Bjoern Haefner and Daniel Cremers}
\institute{Mohammed Brahimi \at Technical University of Munich, \email{mohammed.brahimi@tum.de}
\and Yvain Qu\'eau \at Normandie Univ., UNICAEN, ENSICAEN, CNRS, \email{yvain.queau@ensicaen.fr}
\and Bjoern Haefner \at Technical University of Munich, \email{bjoern.haefner@tum.de}
\and Daniel Cremers \at Technical University of Munich, \email{cremers@tum.de}
}
\begin{document}
\maketitle


\abstract{Uncalibrated photometric stereo aims at estimating the 3D-shape of a surface, given a set of images captured from the same viewing angle, but under unknown, varying illumination. While the theoretical foundations of this inverse problem under directional lighting are well-established, there is a lack of mathematical evidence for the uniqueness of a solution under general lighting. On the other hand, stable and accurate heuristical solutions of uncalibrated photometric stereo under such general lighting have recently been proposed. The quality of the results demonstrated therein tends to indicate that the problem may actually be well-posed, but this still has to be established. The present paper addresses this theoretical issue, considering first-order spherical harmonics approximation of general lighting. Two important theoretical results are established. First, the orthographic integrability constraint ensures uniqueness of a solution up to a global concave-convex ambiguity, which had already been conjectured, yet not proven. Second, the perspective integrability constraint makes the problem well-posed, which generalizes a previous result limited to directional lighting. Eventually, a closed-form expression for the unique least-squares solution of the problem under perspective projection is provided, allowing numerical simulations on synthetic data to empirically validate our findings.}

\section{Introduction}

Among the many photographic techniques which can be considered for the 3D-reconstruction of a still surface, photometric stereo~\cite{Woodham1980} is often considered as a first choice when it comes to the recovery of very thin geometric structures. Nevertheless, the classic formulation of photometric stereo requires illumination to be highly controlled: each image must be captured under a single collimated light source at infinity, and the direction and relative intensity of each source must be calibrated beforehand. In practice, this restricts possible applications of the technique to laboratory setups where collimation of light can be ensured and a (possibly tedious) calibration procedure can be carried out.

Considering uncalibrated general lighting i.e., lighting induced by unknown, non-collimated sources and in the presence of ambient lighting, would both drastically simplify the 3D-scanning process for non-experts, and allow to bring photometric stereo outside of the lab~\cite{Shi2014}. The theoretical foundations of the problem under uncalibrated directional lighting are well-understood: the solution can be recovered only up to a linear transformation~\cite{Hayakawa1994}. When integrability is enforced, this linear ambiguity reduces to the generalized bas-relief one under orthographic projection~\cite{Yuille1997}, and vanishes under perspective projection~\cite{Papadhimitri2013}. This work rather focuses on uncalibrated general lighting represented using first-order spherical harmonics~\cite{Basri2003,Ramamoorthi2001}, in which case the solution can be recovered only up to a Lorentz transformation~\cite{Basri2007} and it has been conjectured - but not proven yet, that additional constraints such as integrability may reduce this ambiguity. One reason for thinking that this conjecture might hold is that stable numerical implementations of uncalibrated photometric stereo under general illumination have been proposed recently, under both orthographic~\cite{Bartal2018,Mo2018} and perspective~\cite{Haefner2019} projections. Despite having no theoretical foundation, the results provided therein do not exhibit a significant low-frequency bias which would reveal an underlying ambiguity: empirically, the problem seems well-posed.

The objective of this paper is thus to establish the uniqueness of a solution to the problem of uncalibrated photometric stereo under general illumination, represented by first-order spherical harmonics. After discussing the classic case of directional lighting in Section~\ref{sec:2}, we characterize in Section~\ref{sec:3} the ambiguities arising in uncalibrated photometric stereo under first-order spherical harmonics lighting. Then, we show in Section~\ref{sec:4} that imposing integrability of the sought normal field resolves such ambiguities. In the orthographic case, only a global concave-convex ambiguity remains, hence the ambiguity is characterized by a single binary degree of freedom. For comparison, in the directional case there are three real degrees of freedom characterizing the \textit{generalized bas-relief} (GBR) ambiguity~\cite{Belhumeur1999}. Moreover, under perspective projection the problem becomes completely well-posed, which generalizes the result of~\cite{Papadhimitri2013} to more general lighting. In this case the solution can even be determined in closed-form, as shown in Section~\ref{sec:5}. Section~\ref{sec:6} eventually recalls our findings, and suggests future research directions.

\section{Preliminaries: photometric stereo under directional lighting}
\label{sec:2}

Assuming a Lambertian surface is observed from a still camera under $m \geq 1$ different directional lighting indexed by $i \in \{1,\dots,m\}$, the graylevel in the $i$-th image can be modeled as follows: 
\begin{equation}
    I^i(x) = \rho(x) \, \mathbf{n}(x)^\top \, \mathbf{l}^i, \qquad \forall x \in \Omega,
    \label{eq:1}
\end{equation}
where $\Omega \subset \R^2$ is the reconstruction domain (projection of the 3D-surface onto the image plane), $\rho(x)>0$ is the albedo at the surface point conjugate to pixel $x$, $\mathbf{n}(x) \in \mathbb{S}^2 \subset \mathbb{R}^3$ is the unit-length outward normal at this point, and $\mathbf{l}^i \in \mathbb{R}^3$ is a vector oriented towards the light source whose norm represents the relative intensity of the source. Photometric stereo consists, given a set of graylevel observations $I^i, i \in \{1,\dots,m\}$, in estimating the shape (represented by the surface normal $\mathbf{n}$) and the reflectance (represented by the surface albedo $\rho$). Depending whether the lighting vectors $\mathbf{l}^i$ are known or not, the problem is called calibrated or uncalibrated. 

\subsection{Calibrated photometric stereo under directional lighting}

Woodham showed in the late 70s~\cite{Woodham1978} that $m \geq 3$ images captured under non-coplanar, known lighting vectors were sufficient to solve this problem. Indeed, defining for every $x \in \Omega$ the following observation vector $\mathbf{i}(x) \in \mathbb{R}^{m}$, lighting matrix $\mathbf{L} \in \mathbb{R}^{m \times 3}$ and surface vector $\mathbf{m}(x) \in \mathbb{R}^3$ :
\begin{equation}
    \mathbf{i}(x) = \begin{bmatrix} I^1(x) \\ \vdots \\ I^m(x) \end{bmatrix}, \qquad 
    \mathbf{L} = \begin{bmatrix} \mathbf{l}^{1 \top} \\ \vdots \\ \mathbf{l}^{m \top} \end{bmatrix}, \qquad 
    \mathbf{m}(x) = \rho(x) \mathbf{n}(x),
    \label{eq:2}
\end{equation}
the set of equations~\eqref{eq:1} can be rewritten as a linear system in $\mathbf{m}(x)$:
\begin{equation}
    \mathbf{i}(x) = \mathbf{L} \, \mathbf{m}(x), \qquad \forall x \in \Omega.
\label{eq:3}
\end{equation}
Provided that $\mathbf{L}$ is of rank three,~\eqref{eq:3} admits a unique least-squares solution in $\mathbf{m}(x)$, from which the normal and albedo can be extracted according to:
\begin{equation}
    \rho(x) = \lvert \mathbf{m}(x) \rvert, \qquad
    \mathbf{n}(x) = \frac{\mathbf{m}(x)}{\lvert \mathbf{m}(x) \rvert}.
\label{eq:4}
\end{equation}

Such a simple least-squares approach may be replaced by robust variational or learning-based strategies to ensure robustness~\cite{Ikehata2018,Li2019,Queau2017,Radow2019}. There also exist numerical solutions for handling non-Lambertian reflectance models~\cite{Chen2019,Khanian2018,Mecca2016,Tozza2016}, non-distant light sources~\cite{Logothetis2016,Mecca2014,Queau2018,Queau2017c}, or the ill-posed cases where $m=2$~\cite{Kozera2018,Mecca2013,Onn1990,Queau2017b} or $m=1$~\cite{Breuss2012,Durou2008,Tozza2016b,Zhang1999}. Such issues are not adressed in the present paper which rather focuses on the theoretical foundations of uncalibrated photometric stereo.

\subsection{Uncalibrated photometric stereo under directional lighting}

The previous strategy relies on the knowledge of the lighting matrix $\mathbf{L}$, and it is not straightforward to extend it to unknown lighting. Let us illustrate this in the discrete setting, denoting the pixels as $x_j$, $j \in \{1,\dots,n\}$ where $n = \left| \Omega \right|$ is the number of pixels, and stack all the observations in an observation matrix $\mathbf{I} \in \mathbb{R}^{m \times n}$ and all the surface vectors in a surface matrix $\mathbf{M} \in \mathbb{R}^{3 \times n}$: 
\begin{equation}
    \mathbf{I} = \begin{bmatrix} \mathbf{i}(x_1), \dots, \mathbf{i}(x_n) \end{bmatrix}, \qquad 
    \mathbf{M} = \begin{bmatrix} \mathbf{m}(x_1), \dots, \mathbf{m}(x_n) \end{bmatrix}.
\end{equation}
Now, the set of $n$ linear systems~\eqref{eq:3} can be represented compactly as: 
\begin{equation}
    \mathbf{I} = \mathbf{L} \mathbf{M}
    \label{eq:6}
\end{equation}
where both the lighting matrix $\mathbf{L}$ and the surface matrix $\mathbf{M}$ are unknown. Since we know that $\mathbf{L}$ should be of rank three, a joint least-squares solution in $(\mathbf{L},\mathbf{M})$ can be computed using truncated singular value decomposition~\cite{Hayakawa1994}. Nevertheless, such a solution is not unique, since given a possible solution $(\mathbf{L},\mathbf{M})$, any couple $(\mathbf{L} \mathbf{A}^{-1}, \mathbf{A} \mathbf{M})$ with $\mathbf{A} \in GL(3,\R)$ is another solution: 
\begin{equation}
    \mathbf{I} = \mathbf{L} \mathbf{M} = \left(\mathbf{L} \mathbf{A}^{-1} \right)\left( \mathbf{A}  \mathbf{M} \right), \qquad \forall \mathbf{A} \in GL(3,\R),
    \label{eq:7}
\end{equation}
or equivalently, in the continuous setting:
\begin{equation}
    \mathbf{i}(x) =  \mathbf{L} \mathbf{m}(x) = \left(\mathbf{L} \mathbf{A}^{-1} \right)\left( \mathbf{A}  \mathbf{m}(x) \right), \qquad \forall (x,\mathbf{A}) \in \Omega \times GL(3,\R). 
    \label{eq:8}
\end{equation}

However, not any surface matrix $\mathbf{M}$ (or $\mathbf{m}$-field, in the continuous setting) is acceptable as a solution. Indeed, this encodes the geometry of the surface, through its normals. Assuming that the surface is regular, its normals should satisfy the so-called integrability (or zero-curl) constraint. This constraint permits to reduce the ambiguities of uncalibrated photometric stereo, as shown in the next two subsections.

\subsection{Integrability under orthographic projection}

Let us assume orthographic projection and denote $\mathbf{n}(x) := \left[n_1(x),n_2(x),n_3(x)\right]^\top$ the surface normal at 3D-point conjugate to pixel~$x$. Let us further represent the surface as a Monge patch i.e., a differentiable mapping $X : \Omega \to \R^3$ of the form $X(x) = (x,z(x))$,
where $z:\Omega \to \mathbb{R}$ is a depth map. Let us assume this map $z$ is twice differentiable, and let $\nabla z(x) = \left[ z_u(x),z_v(x)\right]^\top \in \mathbb{R}^2$ be its gradient in some orthonormal basis $(u,v)$ of the image plane. The integrability constraint is essentially a particular form of Schwarz' theorem, which implies that 
\begin{equation}
     z_{uv} = z_{vu} \qquad \text{over~} \Omega.
    \label{eq:9}
\end{equation}
From the definition  $\mathbf{n}(x) = \dfrac{\begin{bmatrix}
    z_u(x), &
    z_v(x), &
    -1
    \end{bmatrix}^\top}{\sqrt{z_u(x)^2 + z_v(x)^2 +1}} $ of the surface normal, and since $\mathbf{m}(x) = \rho(x) \, \mathbf{n}(x)$, Eq.~\eqref{eq:9} can be rewritten as: 
\begin{equation}
    \left(\frac{m_1}{m_3}\right)_v = \left(\frac{m_2}{m_3}\right)_u \qquad \text{over~} \Omega.
    \label{eq:10}
\end{equation}

Now, let us assume that one has found an $\mathbf{m}$-field solution of the  left-hand side of~\eqref{eq:8}, which further satisfies the integrability constraint~\eqref{eq:10} (in the discrete setting, this can be achieved using matrix factorization~\cite{Yuille1997} or convex optimization techniques~\cite{Sengupta2018}). It can be shown that not all transformations $\mathbf{A}$ in the  right-hand side of~\eqref{eq:8} preserve this constraint. Indeed, the only ones which are acceptable are those of the  generalized bas-relief group. Such matrices define a subgroup of $GL(3,\R)$ under the matrix product, and have the following form~\cite{Belhumeur1999}: 
\begin{equation}
\mathbf{G} = 
\begin{pmatrix}
\lambda & 0 & -\mu \\
      0 & \lambda  &  -\nu \\
      0 & 0 & 1
\end{pmatrix}, \qquad 
\mathbf{G}^{-1} = 
\begin{pmatrix}
1 & 0 & \mu/\lambda \\
      0 & 1  &  \nu/\lambda \\
      0 & 0 & 1
\end{pmatrix}, \qquad  (\lambda, \mu, \nu) \in \R^{3}  \text{ and } \lambda \neq 0.
\end{equation}

The three parameters $\mu,\nu$ and $\lambda$ characterize the GBR ambiguity inherent to uncalibrated photometric stereo under directional illumination and orthographic viewing. Intuitively, they can be understood as follows: any set of photometric stereo images can be reproduced by scaling the surface shape (this is the role of $\lambda$), adding a plane to it (this is the role of $\mu$ and $\nu$), and moving the lighting vectors accordingly. If one is given a prior on the distribution of albedo values, on that of lighting vectors, or on the surface shape, then the three parameters can be estimated i.e., the ambiguity can be resolved. The literature on that particular topic is extremely dense, see e.g.~\cite{Shi2019} for an overview,~\cite{Chen2019b} for a modern numerical solution based on deep learning, and~\cite{Peng2017} for an application to RGB-D sensing. As we shall prove later in Section~\ref{sec:4.1}, in the case of non-directional lighting represented using first-order spherical harmonics, the ambiguity is much simpler since it comes down to a global concave$/$convex one.

\subsection{Integrability under perspective projection}

To terminate this discussion on uncalibrated photometric stereo under directional lighting, let us discuss the case of perspective projection, which was shown to be well-posed by Papadhimitri and Favaro in~\cite{Papadhimitri2013}. In the following, $x = (u,v)$ denotes the pixel coordinates with respect to the principal point (which is the projection of the camera center onto the image plane) and ${f}>0$ denotes the focal length. The surface is now represented as the set of 3D-points $z(x) \left[u/{f},v/{f},1\right]^\top$. Now, let us examine the perspective counterpart of the orthographic integrability constraint~\eqref{eq:10}. \newpage

It is easy to show (see, e.g.,~\cite{Queau2018b}) that the surface normal is now defined as 
 \nopagebreak
\begin{equation}
\label{eq:normal_perspective}
    \mathbf{n}(x) = \frac{1}{ \sqrt{{f}^2 \lvert \nabla z(x) \rvert^2  + \left( -z(x) - \left[u,v\right]^\top \nabla z(x) \right)^2 } }
    \begin{bmatrix}
     {f} \, \nabla z (x) \\
     -z(x) - \left[u,v\right]^\top \nabla z(x)
    \end{bmatrix}.
\end{equation}
If we define the $\log$ depth map as: 
\begin{align}
    & \tilde{z} = \text{log}(z),
\end{align}
and denote:
\begin{align}
    & p = -\dfrac{n_1}{n_3},& & q = -\dfrac{n_2}{n_3}, \label{eq:new14} \\
    & \hat{p} = \dfrac{p}{f - up - vq}, & & \hat{q} = \dfrac{q}{f - up - vq}, \label{eq:new15}
\end{align}
then it is straightforward to show that 
\begin{equation}
     \nabla \tilde{z} = \left[\hat{p},\hat{q}\right]^\top, \label{eq:new16}
\end{equation}
and that Schwarz' theorem~\eqref{eq:9} can be equivalently rewritten in terms of the gradient of the $\log$ depth map:
\begin{equation}
     \tilde{z}_{uv} = \tilde{z}_{vu}. 
    \label{eq:schwarz_persp}
\end{equation}

This equation can be equivalently rewritten in terms of the coefficients of $\mathbf{m} = \rho\,\mathbf{n}$, just as we obtained~\eqref{eq:10} for the orthographic case. This rewriting is given by the following proposition, whose proof can be found in Appendix~\hyperref[sec:app3]{A}:
\begin{proposition}
\label{cor:2}
Let $\mathbf{m} =  \left[m_1,m_2,m_3\right]^\top:\,\Omega \to \R^3$ a field defined as $\mathbf{m}:=\rho \, \mathbf{n}$, with $\rho:\,\Omega \to \mathbb{R}$ an albedo map and $\mathbf{n}:\,\Omega \to \mathbb{S}^2 \subset \mathbb{R}^3$ a normal field. The  normal field $\mathbf{n}$ is integrable iff the coefficients of $\mathbf{m}$ satisfy the following relationship over~$\Omega$:
\begin{align}
  & u(m_{1}m_{2u} - m_{1u}m_{2}) + v(m_{1}m_{2v} - m_{1v}m_{2})\nonumber \\
  &  + f(m_{1}m_{3v} - m_{1v}m_{3}) + f(m_{2u}m_{3} - m_{2}m_{3u}) = 0. 
  \label{eq:perspectiveintegrabilityconstraint}
\end{align}
\end{proposition}

The integrability constraint~\eqref{eq:perspectiveintegrabilityconstraint} is slightly more complicated than the orthographic one~\eqref{eq:10}. Yet, this slight difference is of major importance, because the set of linear transformations $\mathbf{A}$ in~\eqref{eq:8} which preserve this condition is restricted to the identity matrix~\cite{Papadhimitri2013}. This means, under perspective projection and directional lighting the uncalibrated photometric stereo problem is well-posed. As we shall prove later in Section~\ref{sec:4.2}, such a result can actually be extended to more general lighting represented using first-order spherical harmonics. Let us now elaborate on such a modeling of general lighting, and characterize the ambiguities therein.

\section{Characterizing the ambiguities in uncalibrated photometric stereo under general lighting}
\label{sec:3}

The image formation model~\eqref{eq:1} is a simplified model, corresponding to the presence of a single light source located at infinity. However, this assumption is difficult to ensure in real-world experiments, and it would be more convenient to have at hand an image formation model accounting for general lighting (to handle multiple light sources, ambient lighting, etc.). 

\subsection{Spherical harmonics approximation of general lighting}

The most general image formation model for Lambertian surfaces would integrate the incident lighting received from all directions $\mathbf{u}_l \in \mathbb{S}^2$:
\begin{equation}
    I^i(x) = \rho(x) \int_{\mathbb{S}^2} s^i(x,\mathbf{u}_l)\,
    \max \lbrace \mathbf{n}(x)^\top \mathbf{u}_l,0 \rbrace
    \, \mathrm{d}\mathbf{u}_l, \qquad \forall x \in \Omega,
    \label{eq:13}
\end{equation}
where we denote $s^i(x,\mathbf{u}_l) \in \R$ the intensity of the light source in direction $\mathbf{u}_l \in \mathbb{S}^2$ at the surface point conjugate to pixel $x$ in the $i$-th image. In~\eqref{eq:13}, the $\max$ operator encodes self-shadows: it ensures that the amount of reflected light does not take negative values for surface elements not facing the light source.

Assuming a single light source illuminates the scene in the $i$-th image, and neglecting self-shadows, then Eq.~\eqref{eq:13} obviously comes down to the simplified model~\eqref{eq:1}. However, there exist other simplifications of the integral model~\eqref{eq:13}, which allow to handle more general illumination. Namely, the spherical harmonics approximation which were introduced simultaneously in \cite{Basri2003} and~\cite{Ramamoorthi2001}. In the present work we focus on first-order spherical harmonics approximation, which is known to capture approximately $87\%$ of general lighting~\cite{Frolova2004}. Using this approximation,~\eqref{eq:13} simplifies to (see the aforementioned papers for technical details):
\begin{equation}
    I^i(x) = \rho(x) \begin{bmatrix} 1 \\ \mathbf{n}(x) \end{bmatrix}^\top \mathbf{l}^i,\qquad \forall x \in \Omega,
    \label{eq:14}
\end{equation}
with $\mathbf{l}^i \in \R^4$ a vector representing the general illumination in the $i$-th image. Denoting $\mathbf{L} = \left [ \mathbf{l}^1,\hdots,\mathbf{l}^m \right]^\top \in \R^{m \times 4}$ the general lighting matrix, System (\ref{eq:14}) can be rewritten in the same form as the directional one~\eqref{eq:6}: 
\begin{align}
&\mathbf{i}(x) = \mathbf{L} \mathbf{m}(x),\qquad \forall x \in \Omega, \label{eq:15} \\
& ~ \text{with} \qquad  \mathbf{m}(x) = \rho(x) \begin{pmatrix}
1 \\
\mathbf{n}(x)
\end{pmatrix}.
\label{eq:16}
\end{align}

\subsection{Uncalibrated photometric stereo under first-order spherical harmonics lighting}

Uncalibrated photometric stereo under first-order spherical harmonics lighting comes down to solving the set of linear systems~\eqref{eq:15} in terms of both the general ligthing matrix $\mathbf{L}$ and the $\mathbf{m}$-field (which encodes albedo and surface normals). In the directional case discussed previously, this was possible only up to an invertible linear transformation, as shown by~\eqref{eq:8}. Despite appearing more complicated at first glance, the case of first-order spherical harmonics is actually slightly more favorable than the directional one: not all such linear transformations are acceptable, because they have to preserve the particular form of the $\mathbf{m}$-field, given in Eq.~\eqref{eq:16}. That is to say, given one $\mathbf{m}$-field solution and another one $\mathbf{m}^\ast=\mathbf{A} \mathbf{m}$ obtained by applying an invertible linear transformation $\mathbf{A} \in GL(4,\R)$, the entries $c_1, c_2, c_3, c_4$ of $\mathbf{m}^\ast$ should respect the constraint ${c_1}^2 = {c_2}^2 + {c_3}^2 + {c_4}^2$ over $\Omega$ (cf. Eq.~\eqref{eq:16}, remembering that each surface normal has unit length). 

As discussed in~\cite{Basri2007}, this means that ambiguities in uncalibrated photometric stereo under first-order spherical harmonics are characterized as follows:
\begin{equation}
    \mathbf{i}(x) =  \mathbf{L} \mathbf{m}(x) = \left(\mathbf{L} \mathbf{A}^{-1} \right)\left( \mathbf{A}  \mathbf{m}(x) \right), \qquad \forall (x,\mathbf{A}) \in \Omega \times L_s,
    \label{eq:17}
\end{equation}
where $L_s$ is the space of scaled Lorentz transformations defined by 
\begin{equation}
    L_s = \lbrace s\mathbf{A} \hspace{2mm} | \hspace{2mm} s \in \R\backslash \{0\} \text{ and } \mathbf{A} \in L \rbrace, 
    \label{eq:18}
\end{equation}
with $L$ the Lorentz group~\cite{Poincare1905} arising in Einstein's theory of special relativity~\cite{Einstein1905}:
\begin{align}
    & L = \lbrace \mathbf{A} \in GL(4,\R) \hspace{2mm} | \hspace{2mm} \forall \mathbf{x} \in \R^4, \hspace{2mm} l(\mathbf{A}\mathbf{x}) = l(\mathbf{x}) \rbrace,\\
    & \quad \text{~with~}\quad l : (t,x,y,z) \mapsto x^2 + y^2 + z^2 - t^2.
\end{align}

In spite of the presence of the scaled Lorentz ambiguity in Eq.~\eqref{eq:17}, several heuristical approaches to solve uncalibrated photometric stereo under general lighting have been proposed lately. Let us mention the approaches based on hemispherical embedding~\cite{Bartal2018} and on equivalent directional lighting~\cite{Mo2018}, which both deal with the case of orthographic projection, and the variational approach in~\cite{Haefner2019} for that of perspective projection. The empirically observed stability of such implementations tends to indicate that the problem might be better-posed than it seems, as already conjectured in~\cite{Basri2007}. In order to prove this conjecture, we will show in Section~\ref{sec:4} that not all scaled Lorentz transformations preserve the integrability of surface normals. To this end, we need to characterize algebraically a scaled Lorentz transformation.

 \subsection{Characterization of the scaled Lorentz transformation}
 
 We propose to characterize any ambiguity matrix $\mathbf{A} \in L_s$ in~\eqref{eq:17} by means of a scale factor $s \neq 0$ (one degree of freedom), a vector inside the unit $\R^3$-ball $\mathbf{v} \in B(\mathbf{0},1)$ (three degrees of freedom, where $B(\mathbf{0},1) = \left\{ \mathbf{x} \in \R^3,  \lvert x \rvert <1 \right\}$) and a 3D-rotation matrix $\mathbf{O} \in SO(3,\R)$ (three degrees of freedom, hence a total of seven). More explicitly, any scaled Lorentz transformation can be characterized algebraically as follows:
 
 \begin{theorem}
 For any scaled Lorentz transformation $\mathbf{A} \in L_s$, there exists a unique triple $(s,\mathbf{v}, \mathbf{O}) \in \R\backslash \{0\} \times B(\mathbf{0},1) \times SO(3,\R)$ such that: 
\begin{equation}
    \mathbf{A} =
s
\begin{pmatrix}
\epsilon_{1}(\mathbf{A}) \, \gamma & \epsilon_{1}(\mathbf{A}) \, \gamma \, \mathbf{v}^\top\mathbf{O}\\
 \\
\epsilon_{2}(\mathbf{A}) \,\gamma \, \mathbf{v} &\hspace{2mm} \epsilon_{2}(\mathbf{A})\,(\mathbf{I_{3}} + \frac{\gamma^{2}}{1 + \gamma} \mathbf{v} \mathbf{v}^\top) \mathbf{O}  
\end{pmatrix},
\end{equation}
with 
\begin{align}
   & \gamma = \frac{1}{\sqrt{1 - \lvert \mathbf{v} \rvert^2}}, \\
  &  \epsilon_1 (\mathbf{A}) = \left \{
   \begin{array}{r l}
     1 &    \text{~if~}  P_o(\mathbf{A}),  \\
     -1 &  \text{~else},
   \end{array}
   \right. \\
 &  \epsilon_2 (\mathbf{A}) = \left \{
   \begin{array}{r l}
     -1 &      \text{~if~} (P_p(\mathbf{A}) \land \overline{P_o(\mathbf{A})}) \lor (\overline{P_p(\mathbf{A})} \land P_o(\mathbf{A})),   \\
     1 & \text{~else},
   \end{array}
   \right.
\end{align}
and $P_p(\mathbf{A})$ stands for ``$\mathbf{A}$ is proper'', $P_o(\mathbf{A})$ for ``$\mathbf{A}$ is orthochronous'',
\label{thm:1}
\end{theorem}
 where we recall that a Lorentz matrix $\mathbf{A}$ is ``proper'' iff it preserves the orientation of the Minkowski spacetime, and it is ``orthochronous'' iff it preserves the direction of the time, i.e.:
 \begin{align}
     \mathbf{A} \in L \text{~is proper~}  &\Longleftrightarrow \text{det}(\mathbf{A}) > 0, \\
     \mathbf{A} \in L \text{~is orthochronous~}  & \Longleftrightarrow \forall \mathbf{x} = \left[t,x,y,z\right]^\top \in \R^4, \text{sign}(t) = \text{sign}(t'), \nonumber \\
      & \qquad  \text{~where~} \mathbf{Ax} = \left[t',x',y',z'\right]^\top.   
 \end{align}
The opposites are improper and non-orthochronous, and we note $L^{p}_{o}, L^{i}_{o}, L^{p}_{n}$ and $L^{i}_{n}$ the sets of Lorentz transformations which are respectively proper and orthochronous, improper and orthochronous, proper and non-orthochronous, and improper and non-orthochronous. The Lorentz group is the union of all these spaces, i.e. $L = L^{p}_{o} \cup L^{i}_{o} \cup L^{p}_{n} \cup L^{i}_{n}$. 

Using Theorem \ref{thm:1} (whose proof can be found in \hyperref[sec:app1]{Appendix B}) to characterize the underlying ambiguity of uncalibrated photometric stereo under general lighting, we are ready to prove that imposing integrability disambiguates the problem.

\section{Integrability disambiguates uncalibrated photometric stereo under general lighting}
\label{sec:4}

As we have seen in the previous section, uncalibrated photometric stereo under general lighting is ill-posed without further constraints, since it is prone to a scaled Lorentz ambiguity, cf. Eq.~\eqref{eq:17}. Now, let us prove that not all scaled Lorentz transformations preserve the integrability of the underlying normal field. 

{We shall assume through the next two subsections that the pictured surface is twice differentiable and non-degenerate, in a sense which will be clarified in Section~\ref{sec:4.3}}. Then, the only acceptable Lorenz transformation is the one which globally exchanges concavities and convexities in the orthographic case, while it is the identity in the perspective case. That is to say, the orthographic case suffers only from a global concave$/$convex ambiguity, while the perspective one is well-posed.
\subsection{Orthographic case}
\label{sec:4.1}

First, let us prove that under orthographic projection and first-order spherical harmonics lighting, there are only two integrable solutions to uncalibrated photometric stereo, and they differ by a global  concave$/$convex transformation. 

To this end, we consider the genuine solution $\mathbf{m}(x)$ of~\eqref{eq:15} corresponding to a normal field  $\mathbf{n}(x)$ and albedo map $\rho(x)$, and another possible solution $\mathbf{m}^\ast(x) = \mathbf{A} \mathbf{m}(x),~\mathbf{A} \in L_s$, with $(\rho^\ast(x),\mathbf{n}^\ast(x))$ the corresponding albedo map and surface normals. The pictured surface being twice differentiable, the genuine normal field $\mathbf{n}$ is integrable by construction. We establish in this subsection that if the other candidate normal field $\mathbf{n}^\ast$ is assumed integrable as well, then both the genuine and the alternative solutions differ according to:
\begin{align}
    \left \{
   \begin{array}{r c l}
      \rho^\ast(x) & = & \alpha \, \rho_j(x) \\[4pt]
      n_{1}^\ast(x) & = & \lambda \, n_{1}(x)\\[4pt]
      n_{2}^\ast(x) & = & \lambda \, n_{2}(x)\\[4pt]
      n_{3}^\ast(x) & = & n_{3}(x)
   \end{array}
   \right.,\qquad \forall x \in \Omega,
   \label{ambiguityNormalFOOrth}
\end{align}
where $\alpha >0$ and $\lambda \in \{-1,1\}$. That is to say, all albedo values are globally scaled by the same factor $\alpha$, while the sign of the first two components of all normal vectors are inverted i.e., concavities are turned into convexities and vice-versa. The global scale on the albedo should not be considered as an issue, since such values are relative to the camera response function and the intensities of the light sources, and they can be manually scaled back in a post-processing step if needed. However, the residual global concave$/$convex ambiguity shows that shape inference remains ill-posed. Still, the ill-posedness is characterized by a single binary degree of freedom, which is to be compared with the three real degrees of freedom characterizing the GBR ambiguity arising in the case of directional lighting~\cite{Belhumeur1999}.

More formally, this result can be stated as the following theorem, which characterizes the scaled Lorentz transformations in~\eqref{eq:17} preserving the integrability of the underlying normal field:
\begin{theorem}
Under orthographic projection, the only scaled Lorentz transformation $\mathbf{A} \in L_s$ which preserves integrability of normals is the following one, where $\alpha >0$ and $\lambda \in \{-1,1\}$:
\begin{equation}
     \mathbf{A} = \alpha \begin{bmatrix} 
    1 & 0 & 0 & 0 \\
    0 & \lambda & 0 & 0 \\
    0 & 0 & \lambda & 0 \\
    0 & 0 & 0 & 1
    \end{bmatrix}.
\end{equation}
\label{thm:2}
\end{theorem}

\begin{proof}
Let $\mathbf{m}:\,\Omega \to \R^4$ a field with the form of Equation~\eqref{eq:16}, and let $\rho$ and $\mathbf{n}$ the corresponding albedo map and normal field, assumed integrable. The normal field $\mathbf{n}$ being integrable, $p  = -\dfrac{n_1}{n_3}$ and $q = -\dfrac{n_2}{n_3}$ satisfy the integrability constraint $p_v = q_u$ over $\Omega$. Denoting by $(c_1,c_2,c_3,c_4)$ the four components of the field $\mathbf{m}$, and using the expression~\eqref{eq:16} of $\mathbf{m}$, this implies: 
\begin{align}
 & \left(\dfrac{c_{2}}{c_{4}}\right)_v = \left(\dfrac{c_{3}}{c_{4}}\right)_u \qquad \text{over~} \Omega, \nonumber\\[1pt]
\iff & \dfrac{c_{2v}c_4 - c_2c_{4v}}{{c_4}^2} = \dfrac{c_{3u}c_4 - c_3c_{4u}}{{c_4}^2} \qquad \text{over~} \Omega, \nonumber\\[1pt]
\iff & {(c_{2v} - c_{3u})c_4 + c_{4u}c_3 - c_{4v}c_2 = 0 \qquad \text{over~} \Omega.}
\label{eq:integConstraintOrtho}
\end{align}

Let $\mathbf{m}^\ast = \mathbf{A} \mathbf{m}$, with $\mathbf{A}$ a scaled Lorentz transformation having the form given by Theorem~\ref{thm:1}, and let $\rho^\ast$ and $\mathbf{n}^\ast$ the corresponding albedo map and normal field, assumed integrable. The same rationale as above on the alternative normal field $\mathbf{n}^\ast$ yields:
\begin{equation}
\label{eq:integConstraintOrthoDouble}
      (c_{2v}^\ast - c_{3u}^*)c_4^* + c_{4u}^*c_3^* - c_{4v}^*c_2^* = 0 \qquad \text{over~} \Omega.
\end{equation}
Since $\mathbf{m}^* = \mathbf{A}\mathbf{m}$,~\eqref{eq:integConstraintOrthoDouble} writes as:
\begin{multline}
\left(A_{21}c_{1v} + A_{22}c_{2v} + A_{23}c_{3v} + A_{24}c_{4v} - A_{31}c_{1u} - A_{32}c_{2u} \right.\\
 \left.- A_{33}c_{3u} - A_{34}c_{4u} \right)\left(A_{41}c_{1} + A_{42}c_{2} + A_{43}c_{3} + A_{44}c_{4}\right)\\
+ \left(A_{41}c_{1u} + A_{42}c_{2u} + A_{43}c_{3u} + A_{44}c_{4u}\right)\left(A_{31}c_{1} + A_{32}c_{2} + A_{33}c_{3} + A_{34}c_{4}\right) \\
- \left(A_{41}c_{1v} + A_{42}c_{2v} + A_{43}c_{3v}  + A_{44}c_{4v}\right)\left(A_{21}c_{1} + A_{22}c_{2} + A_{23}c_{3} + A_{24}c_{4}\right) \\
= 0 \quad \text{~over~} \Omega.
\label{eq:eezazeaeqd}
\end{multline}
Let us introduce the following notation, $1 \leq i < j \leq 4$, and $k \in \lbrace u,v \rbrace$:  
\begin{equation}
\label{notationVect}
c^{i,j}_k(x) = c_{j}(x) c_{ik}(x) - c_{i}(x) c_{jk}(x),\qquad \forall x \in \Omega,
\end{equation}
and denote as follows the minors of size two of matrix $\mathbf{A}$: 
\begin{equation}
A^{i,j}_{k,l} =  A_{ij}A_{kl} - A_{kj}A_{il},\qquad 1 \leq i < k \leq 4,~1 \leq j < l \leq 4.
\end{equation}
 
Then, factoring~\eqref{eq:eezazeaeqd} firstly by the coefficients $A_{ij}$ and after by $c^{i,j}_u$ and $c^{i,j}_v$ for every $(i,j) \in \lbrace {1,2,3,4} \rbrace$ with $i<j$, we get: 
\begin{multline}
    c^{1,2}_v A^{2,1}_{4,2} +   c^{1,3}_v A^{2,1}_{4,3} +   c^{1,4}_v A^{2,1}_{4,4} +  c^{2,3}_v A^{2,2}_{4,3}
     + c^{2,4}_v A^{2,2}_{4,4} + c^{3,4}_v A^{2,3}_{4,4} \\
- c^{1,2}_u A^{3,1}_{4,2}  -  c^{1,3}_u A^{3,1}_{4,3} -  c^{1,4}_u A^{3,1}_{4,4} -  c^{2,3}_u A^{3,2}_{4,3} -  c^{2,4}_u A^{3,2}_{4,4} -  c^{3,4}_u A^{3,3}_{4,4} = 0 \text{~over~}\Omega.
    \label{eq:minorsOrtho}
\end{multline}

In addition,~\eqref{eq:integConstraintOrtho} also writes as:
\begin{equation}
    c^{2,4}_v = c^{3,4}_u {\qquad \text{~over~} \Omega.}
\end{equation}
Thus, substituting $c^{2,4}_v$ by $c^{3,4}_u$, Eq.~\eqref{eq:minorsOrtho} can be rewritten as:
\begin{equation}
    \mathbf{i}_o(x)^\top \mathbf{a} = 0, \qquad \forall x \in \Omega,
    \label{eq:uneLigne}
\end{equation}
where $\mathbf{i}_o(x) \in \R^{11}$ is the ``orthographic integrability vector'' containing factors $c^{i,j}_u(x)$ and $c^{i,j}_v(x)$, and $\mathbf{a} \in \R^{11}$ contain the minors $A^{i,j}_{k,l}$ of $\mathbf{A}$ appearing in~\eqref{eq:minorsOrtho}. 

Since the surface is assumed to be non-degenerate (cf. Section~\ref{sec:4.3}), there exist at least  $11$ points $x \in \Omega$ such that a full-rank matrix can be formed by concatenating the vectors $\mathbf{i}_o(x)^\top$ row-wise. We deduce that the only solution to~\eqref{eq:uneLigne} is $\mathbf{a}=0$, which is equivalent to the following equations: 
\begin{equation}
\left \{
   \begin{array}{l}
    	A^{3,2}_{4,3} = A^{3,2}_{4,4} = A^{2,3}_{4,4} = A^{2,2}_{4,3} = 0,\\[5pt]
     	A^{3,3}_{4,4} = A^{2,2}_{4,4},\\[5pt]
	A^{2,1}_{4,2} = A^{3,1}_{4,2} = A^{2,1}_{4,3} = A^{3,1}_{4,3} = A^{2,1}_{4,4} = A^{3,1}_{4,4} = 0.
   \end{array}
   \right .
   \label{eq:premierSysteme}
\end{equation}

According to Corollary \ref{cor:1} provided in \hyperref[sec:app2]{Appendix C}, this implies that the submatrix of $\mathbf{A}$ formed by its last three rows and columns is a scaled generalized bas-relief transformation, i.e.:
{there exists a unique quadruple} $ (\lambda, \mu, \nu, \beta) \in \R^4$ with $\lambda \neq 0, \beta \neq 0$, such that:
\begin{equation}
 \mathbf{A} = \begin{pmatrix}
A_{11} & A_{12} & A_{13} & A_{14} \\
A_{21} & \beta\lambda & 0 & -\beta\mu \\
A_{31} & 0 & \beta\lambda & -\beta\nu   \\
A_{41} & 0 & 0 & \beta 
\end{pmatrix}.
\end{equation}
By taking into account the last equation of System~\eqref{eq:premierSysteme}, we get :
\begin{align}
\label{eq:expressionIntermediaireAOrtho}
\mathbf{A} = \begin{pmatrix}
A_{11} & A_{12} & A_{13} & A_{14} \\
0 & \beta\lambda & 0 & -\beta\mu \\
0 & 0 & \beta\lambda & -\beta\nu   \\
0 & 0 & 0 & \beta 
\end{pmatrix}.
\end{align}
Identifying~\eqref{eq:expressionIntermediaireAOrtho} with the expression in Theorem~\ref{thm:1}, $\mathbf{v} = \mathbf{0}$, $\gamma=1$ and $s\,\epsilon_2(\mathbf{A})\,\mathbf{O}=\beta\begin{pmatrix}
\lambda & 0 & -\mu \\
 0 & \lambda & -\nu   \\
0 & 0 & 1 
\end{pmatrix}$. In addition, $\mathbf{O} \in SO(3,\R)$, which implies $\mathbf{O}^\top\mathbf{O} = \mathbf{I_3}$. Thus, since $\epsilon_{2}(\mathbf{A})^2 = 1$, we have: $(s\,\epsilon_{2}(\mathbf{A})\,\mathbf{O})^\top\,\epsilon_{2}(\mathbf{A})\,\mathbf{O}) = s^2\, \mathbf{I_3}$. Equivalently:
\begin{align}
   & \beta\,\begin{pmatrix}
\lambda & 0 & 0 \\
 0 & \lambda & 0   \\
-\mu & -\nu & 1
\end{pmatrix}\,
\beta\,\begin{pmatrix}
\lambda & 0 & -\mu \\
 0 & \lambda & -\nu   \\
0 & 0 & 1  
\end{pmatrix}
 =
\begin{pmatrix}
s^2 & 0 & 0 \\
 0 & s^2 & 0   \\
0 & 0 & s^2 
\end{pmatrix}, \nonumber\\
\iff & \beta^2\begin{pmatrix}
\lambda^2 & 0 & -\mu \lambda \\
 0 & \lambda^2 & -\nu \lambda   \\
-\mu \lambda & -\nu \lambda & \mu^{2} + \nu^{2} + 1 
\end{pmatrix}
 = 
\begin{pmatrix}
s^2 & 0 & 0 \\
 0 & s^2 & 0   \\
0 & 0 & s^2 
\end{pmatrix},
\end{align}
which implies $\lambda^2 = 1, \mu = 0, \nu = 0, \beta^2 = s^2$.

Finally, $\text{det}(s\, \epsilon_{2}(\mathbf{A})\,\mathbf{O}) = \beta \, \lambda^2  \, \beta = \epsilon_{2}(\mathbf{A}) \, s$, thus according to \eqref{eq:expressionIntermediaireAOrtho}:
\begin{align}
\label{eq:expressionACasiFinaleOrtho}
\mathbf{A} = 
s\begin{pmatrix}
\epsilon_1(\mathbf{A}) & 0 & 0 & 0 \\
0 & \epsilon_2(\mathbf{A})\lambda & 0 & 0 \\
0 & 0 & \epsilon_2(\mathbf{A})\lambda & 0  \\
0 & 0 & 0 & \epsilon_2(\mathbf{A})
\end{pmatrix}.
\end{align}

Plugging~\eqref{eq:expressionACasiFinaleOrtho} into $\mathbf{m}^\ast = \mathbf{A} \mathbf{m}$, we obtain:
\begin{align}
    \begin{pmatrix}
    \rho^*(x) \\
    \rho^*(x) \, n_{1}^*(x) \\
    \rho^*(x) \, n_{2}^*(x) \\
    \rho^*(x) \, n_{3}^*(x) 
    \end{pmatrix} = s \begin{pmatrix}
    \epsilon_1(\mathbf{A}) \, \rho(x) \\
    \epsilon_2(\mathbf{A}) \, \lambda \, \rho(x) \,  n_{1}(x) \\
    \epsilon_2(\mathbf{A}) \, \lambda \, \rho(x) \,  n_{2}(x) \\
    \epsilon_2(\mathbf{A}) \, \rho(x) \,  n_{3}(x) 
    \end{pmatrix}, \qquad \forall x \in \Omega.
    \label{eq:equationsSignesOrtho}
\end{align}
Now, knowing that albedos $\rho,\rho^* \geqslant 0$ (they represent the proportion of light which is reflected by the surface), and that the last component of normals $n_{3},n^*_{3} \leq 0$ (the normals point toward the camera), Eq.~\eqref{eq:equationsSignesOrtho} implies that $\epsilon_1(\mathbf{A})$ and $\epsilon_2(\mathbf{A})$ have exactly the same sign as $s$. 

Two cases must be considered. If $s > 0$, then $\epsilon_1(\mathbf{A}) = \epsilon_2(\mathbf{A}) = 1$, and plugging these values into~\eqref{eq:expressionACasiFinaleOrtho} we obtain the expression provided in Theorem~\ref{thm:2}. If $s < 0$, then $\epsilon_1(\mathbf{A}) = \epsilon_2(\mathbf{A}) = -1$, and we again get the expression provided in Theorem~\ref{thm:2}. 
\end{proof}


From a practical point of view, once an integrable normal field candidate has been found heuristically, using e.g. hemispherical embedding~\cite{Bartal2018} or an equivalent directional lighting model~\cite{Mo2018}, the residual ambiguity i.e., the sign of $\lambda$, needs to be set manually, as proposed for instance in in~\cite{Mo2018}. As we shall see now, in the case of perspective projection the problem becomes even completely well-posed, which circumvents the need for any manual intervention. 

\newpage

\subsection{Perspective case}
\label{sec:4.2}

Now we will prove that uncalibrated photometric stereo under first-order spherical harmonics lighting and perspective projection is well-posed. This means, imposing integrability restricts the admittible ambiguity matrices $\mathbf{A}$ in~\eqref{eq:17} to the identity matrix (up to a factor scaling all albedo values without affecting the geometry):
\begin{theorem}
\label{thm:3}
Under perspective projection, the only scaled Lorentz transformation $\mathbf{A} \in L_s$ which preserves integrability of normals is the identity matrix, up to a scale factor $\alpha > 0$:
\begin{equation}
    \mathbf{A} = \alpha \mathbf{I}_4.
\end{equation}
\end{theorem}

\begin{proof}
Let $\mathbf{m}:\,\Omega \to \R^4$ a field with the form of Equation~\eqref{eq:16}, whose normal field is integrable. Let $\mathbf{m}^\ast = \mathbf{A} \mathbf{m}$ another such field whose normal field is integrable, with $\mathbf{A} \in L_s$ a scaled Lorentz transformation having the form given by Theorem~\ref{thm:1}.

Let us denote by $(c_1,c_2,c_3,c_4)$ the four components of the field $\mathbf{m}$, and by $(c_1^\ast,c_2^\ast,c_3^\ast,c_4^\ast)$ those of $\mathbf{m}^\ast$. According to Proposition \ref{cor:2}, the integrability constraint of the normal field associated with $\mathbf{m}^\ast$ writes as follows: 
\begin{equation}
\label{eq:compactIntegrabilityPerspective}
    u (c^\ast)^{2,3}_u + v (c^\ast)^{2,3}_v + f (c^\ast)^{2,4}_v - f (c^\ast)^{3,4}_u = 0 { \qquad \text{~over~}\Omega,} 
\end{equation}
with the same notations as in~\eqref{notationVect}.

As in the previous proof of Theorem \ref{thm:2}, we substitute in the integrability constraint~\eqref{eq:compactIntegrabilityPerspective} the entries of $\mathbf{m}^\ast = \mathbf{A} \mathbf{m}$ with their expressions in terms of entries of $\mathbf{A}$ and $\mathbf{m}$. Then, by factoring firstly by the coefficients $A_{ij}$ and then by $c^{i,j}_u$ and $c^{i,j}_v$ for every $(i,j) \in \lbrace {1,2,3,4} \rbrace$ with $i<j$, we get:
\begin{align}
    & \left(u c^{1,2}_u + v c^{1,2}_v\right) A^{2,1}_{3,2} + \left(u c^{1,3}_u + v c^{1,3}_v\right) A^{2,1}_{3,3} \nonumber \\
+ & \left(u c^{1,4}_u + v c^{1,4}_v\right) A^{2,1}_{3,4} + \left(u c^{2,3}_u + v c^{2,3}_v\right) A^{2,2}_{3,3} \nonumber \\
+ & \left(u c^{2,4}_u + v c^{2,4}_v\right) A^{2,2}_{3,4} + \left(u c^{3,4}_u + v c^{3,4}_v\right) A^{2,3}_{3,4} \nonumber \\
+ & f\left(c^{1,2}_v A^{2,1}_{4,2} + c^{1,3}_v A^{2,1}_{4,3} + c^{1,4}_v A^{2,1}_{4,4}   + c^{2,3}_v A^{2,2}_{4,3} + c^{2,4}_v A^{2,2}_{4,4} + c^{3,4}_v A^{2,3}_{4,4} \right) \nonumber \\
-  & f\left(c^{1,2}_u A^{3,1}_{4,2} + c^{1,3}_u A^{3,1}_{4,3} + c^{1,4}_u A^{3,1}_{4,4}   + c^{2,3}_u A^{3,2}_{4,3} + c^{2,4}_u A^{3,2}_{4,4} + c^{3,4}_u A^{3,3}_{4,4} \right) \nonumber \\
& = 0 {\qquad \text{~over~}\Omega.}   \label{minorsPersp}
\end{align}

By concatenating equations (\ref{minorsPersp}) for all pixels $x \in \Omega$, we get the following set of linear systems:
\begin{equation}
\label{eq:integMatrix_continuous}
    \mathbf{i}_p(x)^\top\mathbf{w} = 0, \qquad \forall x \in \Omega,
\end{equation}
where $\mathbf{w} \in \R^{18}$ contains all the minors $A^{i,j}_{k,l}$ of the ambiguity matrix $\mathbf{A}$ in Eq.\eqref{minorsPersp}, and the ``perspective integrability'' vector $\mathbf{i}_p(x)$ depends only on $u$, $v$, $f$ and $c^{i,j}_k$ i.e., known quantities. We will see later in Section~\ref{sec:5}, that numerically solving the set of equations~\eqref{eq:integMatrix_continuous} provides a simple way to numerically solve uncalibrated perspective photometric stereo under first-order spherical harmonics lighting. 

If in addition we use the fact that $\mathbf{m}$ fulfills the integrability constraint~\eqref{eq:compactIntegrabilityPerspective}, we can substitute $ (u c^{2,3}_u + v c^{2,3}_v)$ by $(-f c^{2,4}_v + f c^{3,4}_u)$ in Eq.~\eqref{minorsPersp}, and we get $17$ summands instead of $18$, turning~\eqref{eq:integMatrix_continuous} as follows
\begin{equation}\label{eq:sisilafamille}
    \mathbf{c}(x)^\top \mathbf{a} = 0 { \qquad \text{~over~}\Omega,}
\end{equation}
where $\mathbf{c}(x),\mathbf{a} \in \R^{17}$.

Since the surface is assumed to be non-degenerate (cf. Section~\ref{sec:4.3}), there exist at least $17$ points $x \in \Omega$ such that a full-rank matrix can be formed by row-wise concatenation of vectors $\mathbf{c}(x)^\top,\,x \in \Omega$. We deduce that $\mathbf{a}$ = 0 and we get the following equations:
\begin{equation}
\label{eq:systemMinorsPersp}
\left \{
   \begin{array}{l}
     A^{3,2}_{4,3} = A^{3,2}_{4,4} = A^{2,3}_{4,4} = A^{2,2}_{4,3} = 0,\\
     A^{2,2}_{3,3} = A^{3,3}_{4,4},\\
     A^{2,2}_{3,3} = A^{2,2}_{4,4},\\
     A^{2,1}_{3,2} =  A^{2,1}_{3,3} = A^{2,1}_{3,4} = A^{2,2}_{3,4} = A^{2,3}_{3,4} = 0,\\
     A^{2,1}_{4,2} = A^{2,1}_{4,3} = A^{2,1}_{4,4} = A^{2,2}_{4,3} = A^{2,3}_{4,4} = 0,\\
     A^{3,2}_{4,3} = A^{3,2}_{4,4} = A^{3,1}_{4,2} = A^{3,1}_{4,3} = A^{3,1}_{4,4} = 0.
   \end{array}
   \right .
\end{equation}

According to the first three equations of System (\ref{eq:systemMinorsPersp}), the submatrix of $\mathbf{A}$ formed by the last three rows and columns is a scaled generalized bas-relief transformation (see Corollary \ref{cor:1} in \hyperref[sec:app2]{Appendix C}). That is to say, there exists a unique quadruple $(\lambda, \mu, \nu, \alpha) \in \R^4 $ with $ \lambda \neq 0, \alpha \neq 0$, such that: 
\begin{equation}
\mathbf{A} = \begin{pmatrix}
A_{11} & A_{12} & A_{13} & A_{14} \\
A_{21} & \alpha\lambda & 0 & -\alpha\mu \\
A_{31} & 0 & \alpha\lambda & -\alpha\nu   \\
A_{41} & 0 & 0 & \alpha 
\end{pmatrix}.
\end{equation}

Taking into account the other equations of system \eqref{eq:systemMinorsPersp}, we get $\lambda = 1$, $\mu = \nu = A_{21} = A_{31} = A_{41} = 0$. Then, the same arguments as those used around Eq.~\eqref{eq:expressionACasiFinaleOrtho} yield $A_{12} = A_{13} = A_{14} = 0$, and the form~\eqref{eq:16} of { $\mathbf{m}^\ast = \mathbf{A}\mathbf{m}$} implies $A_{11} = \alpha$,  which concludes the proof.
\end{proof}

Let us remark that such a particular form of a scaled Lorentz transformation only scales all albedo values, leaving the geometry unchanged. From a practical point of view, this means that once an integrable candidate has been found, it corresponds to the genuine surface and there is no need to manually solve any ambiguity, unlike in the orthographic case. In Section~\ref{sec:5}, we will see that such a candidate can be estimated in closed-form in the discrete setting. This will allow us to empirically verify the validity of our theoretical results, through numerical experiments on simulated images. Before that, let us briefly elaborate on degenerate surfaces i.e., surfaces for which the two theorems in the present section do not hold. 
\newpage
\subsection{Degenerate surfaces}
\label{sec:4.3}

The two previous theorems rely on the assumption that the surface is non ``degenerate''. Although degenerate surfaces are rarely encountered in practice, this notion needs to be clarified for the completeness of this study.

Degenerate surfaces are those having a particularly simple shape, which causes the matrix formed by concatenation of the integrability vectors ($\mathbf{i}_o(x)$ in the orthographic case, cf.~\eqref{eq:uneLigne}, or $\mathbf{c}(x)$ in the perspective case, cf.~\eqref{eq:sisilafamille}) not to be full-rank. Here we algebraically characterize such surfaces, for which the integrability constraint is not enough to solve the Lorentz ambiguity. 

\subsubsection{Orthographic case}

Let $\mathbf{m} : \Omega \to \R^4$ be a field of the form of \eqref{eq:16}, and let $\rho$ and $\mathbf{n}$ be the corresponding albedo map and normal field, respectively. We denote by $(c_1,c_2,c_3,c_4)$ the four components of the field $\mathbf{m}$, and use the definition~\eqref{notationVect} of the coefficients $c^{i,j}_k$, $i,j \in \{1,\dots,4\}$, $k \in \{u,v\}$. Then, the surface defined by the field $\mathbf{m}$ is degenerate iff the $\left(c^{i,j}_k \right)_{(k,i,j)\neq(v,2,4)}$ are linearly dependent, i.e. if there exists a non-zero vector $(\lambda^{i,j}_k ) \in \R^{11}\backslash \{0\}$ such that for any pixel $x\in \Omega$
\begin{equation}
    \sum_{\substack{k \in \lbrace u,v \rbrace \\ 1 \leq i < j \leq 4 \\ (k,i,j)\neq(v,2,4)}} \lambda^{i,j}_k c^{i,j}_k(x) = 0.
\end{equation}

To illustrate this notion on some examples, let us remark that by definition of the coefficients $\left(c^{i,j}_k \right)$:
\begin{equation}
\label{eq:relationNormalsCfields}
    \rho\,\mathbf{n} \times \rho\,\mathbf{n}_u = \begin{pmatrix}
-c^{3,4}_u \\
c^{2,4}_u\\
-c^{2,3}_u
\end{pmatrix}, \qquad \rho\,\mathbf{n} \times \rho\,\mathbf{n}_v = \begin{pmatrix}
-c^{3,4}_v \\
c^{2,4}_v\\
-c^{2,3}_v
\end{pmatrix} \qquad \text{~over~} \Omega,
\end{equation}
where $\times$ denotes the cross-product.

Therefore, the following sufficient (but not necessary) conditions to be a degenerate surface can be formulated:
\begin{itemize}
    \item $\mathbf{n}_u = \mathbf{n}_v = 0$: a planar surface.
    \item $\mathbf{n}_u = 0$ and $\mathbf{n}_v \neq 0$: a surface with vanishing curvature along $u$ (see Figure (\ref{subfig:degenerate1})) ;
    \item $\mathbf{n}_u \neq 0$ and $\mathbf{n}_v = 0$: a surface with vanishing curvature along $v$ (see Figure (\ref{subfig:degenerate2})) ;
    \item $\mathbf{n}_u = \mathbf{n}_v$: a surface with vanishing curvature along $u = -v$ (see Figure (\ref{subfig:degenerate3})) ;
    \item $\mathbf{n}_u = -\mathbf{n}_v$, a surface with vanishing curvature along $u = v$ (see Figure (\ref{subfig:degenerate4})).
\end{itemize}

\begin{figure}[!ht]
    \centering
  \begin{subfigure}[b]{0.45\textwidth}
    \includegraphics[width=\textwidth]{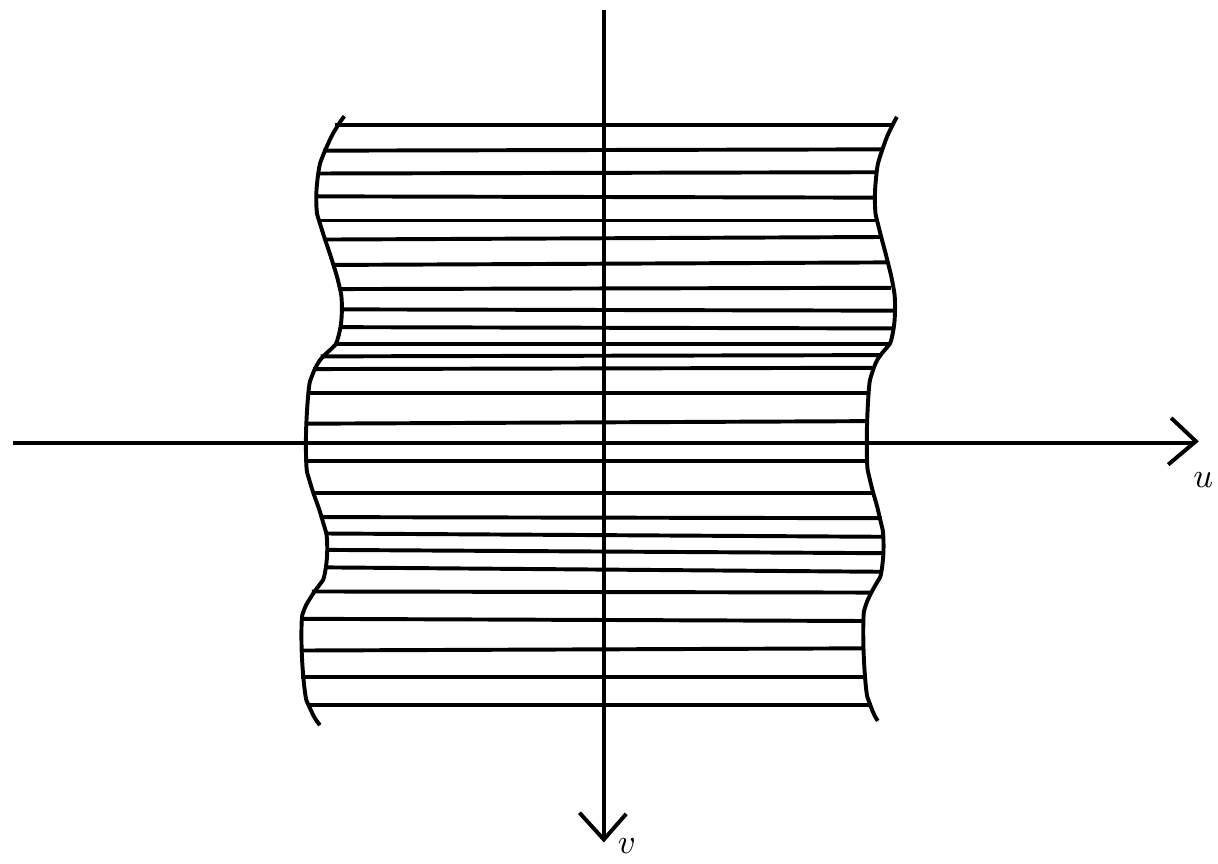}
    \caption{\label{subfig:degenerate1}$\mathbf{n}_u = 0,\mathbf{n}_v \neq 0$ }
  \end{subfigure}
\qquad
  \begin{subfigure}[b]{0.45\textwidth}
    \includegraphics[width=\textwidth]{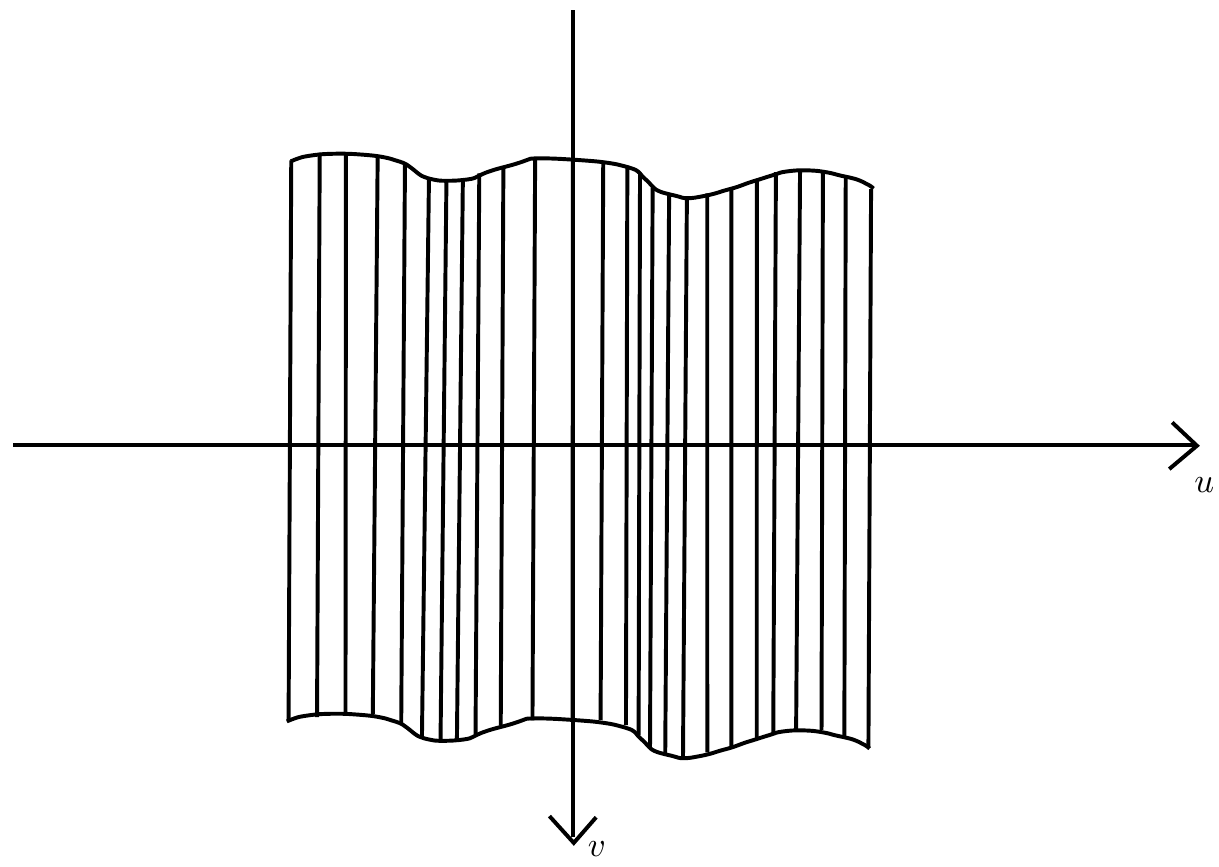}
    \caption{$\mathbf{n}_u \neq 0,\mathbf{n}_v = 0$ \label{subfig:degenerate2}}
  \end{subfigure}
  \begin{subfigure}[b]{0.45\textwidth}
    \includegraphics[width=\textwidth]{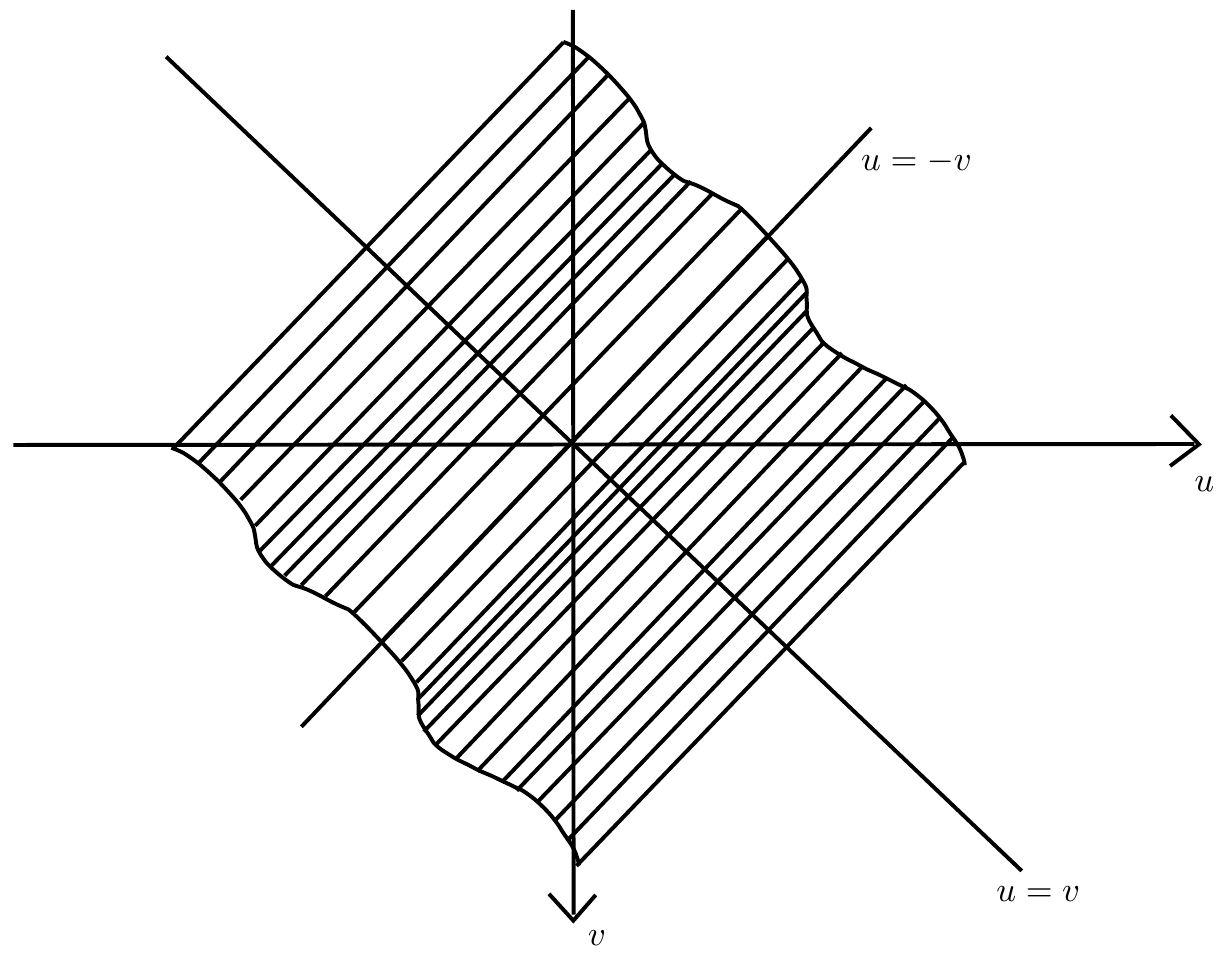}
    \caption{$\mathbf{n}_u = \mathbf{n}_v$ \label{subfig:degenerate3}}
  \end{subfigure}
  \qquad
  \begin{subfigure}[b]{0.45\textwidth}
    \includegraphics[width=\textwidth]{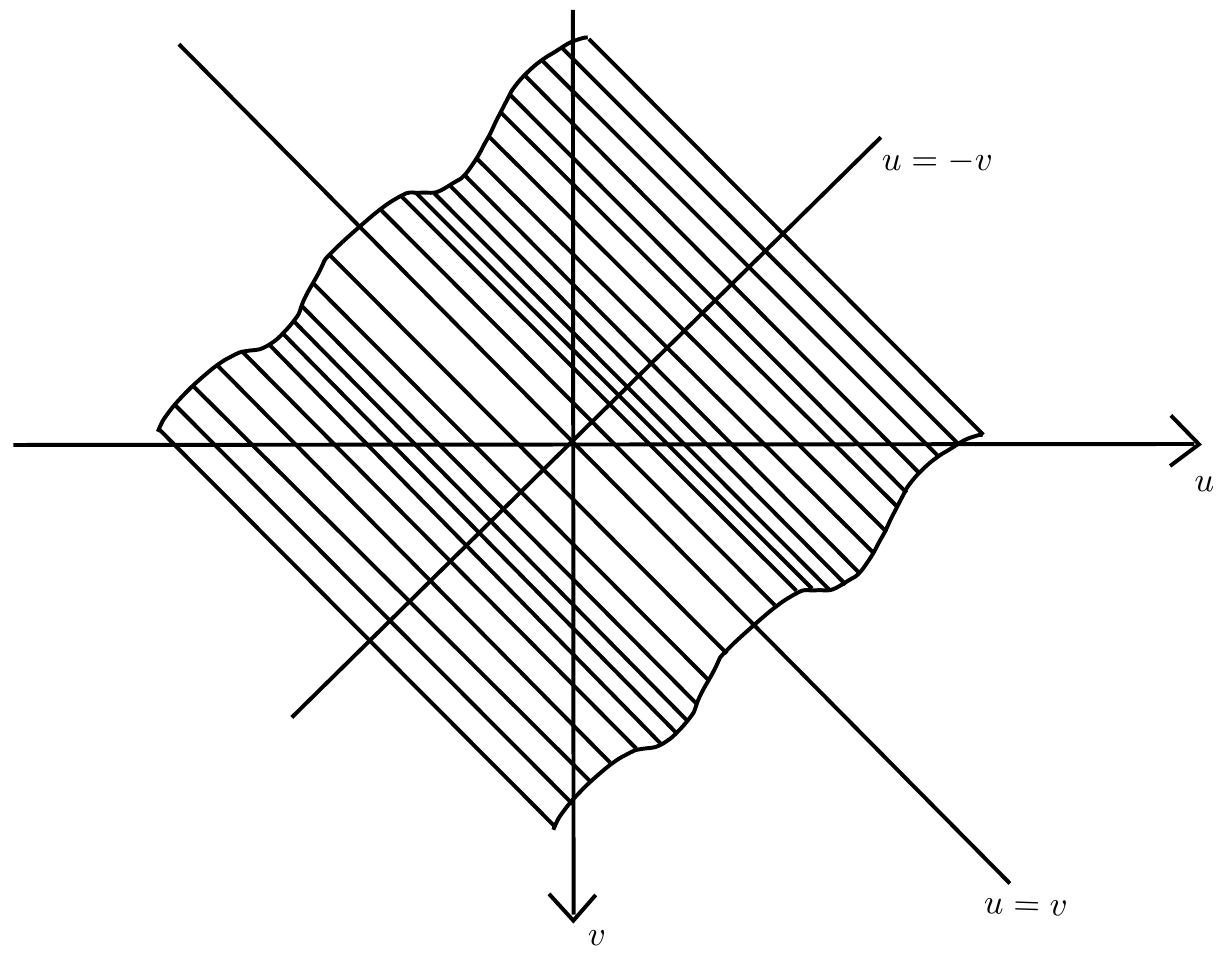}
    \caption{$\mathbf{n}_u = -\mathbf{n}_v$ \label{subfig:degenerate4}}
  \end{subfigure}
  \caption{Examples of degenerate surfaces in the orthographic case.}
  \label{fig:degenerate}
\end{figure}

\subsubsection{Perspective case}

Analogously, a surface is degenerate under perspective projection iff there exists a non-zero vector $\left( (\alpha^{i,j}_k), (\beta_{i,j})_{(i,j) \neq (2,3)} \right) \in \R^{17}\backslash \{0\}$ such that, for any pixel $x = (u,v) \in \Omega$:
\begin{equation}
    \left [\sum_{k \in \lbrace u,v \rbrace}
    \sum_{1 \leq i < j \leq 4} \left (\alpha^{i,j}_k  \right) f c^{i,j}_k(x) \right ]
    + 
    \sum_{\substack{1 \leq i < j \leq 4 \\ (i,j)\neq(2,3)}} \beta_{i,j}\left(u c^{i,j}_u(x) + v c^{i,j}_v(x) \right) = 0,
\end{equation}
where $f$ is the focal length.

The surfaces shown in Figure~\ref{fig:degenerate} are examples of degenerate surfaces. There exist other examples, yet it is not straightforward to characterize them in a simple way. On the other hand, in practice if the surface was simple enough to yield degeneracy, one would not resort to photometric stereo at all. In real-world problems, the geometry of the pictured surface is rich enough, so degenerate surfaces rarely or never arise. This means, it is possible to numerically solve equations such as~\eqref{eq:integMatrix_continuous} in a stable manner. As shown in the next section, this provides a practical way to numerically solve perspective uncalibrated photometric stereo under general lighting. 

\section{Numerical solving of the perspective case}
\label{sec:5}

In this section, we derive a practical algorithm for solving perspective uncalibrated photometric stereo under first-order spherical harmonics lighting. More specifically, we provide a closed-form solution for an integrable normal field satisfying the image formation model~\eqref{eq:15}, provided that the perspective camera is calibrated (i.e., its focal length and principal point are known). 

\subsection{Discrete formulation}

First, let us reformulate the problem in the discrete setting. Let us stack all the observations $I^i(x)$, $i \in \{1,\dots,m\}$, $x \in \Omega$, in a matrix $\mathbf{I} \in \R^{m \times n}$, with $n = \lvert \Omega \rvert$ the number of pixels. Similarly to the directional lighting case represented by~\eqref{eq:6}, the set of linear systems~\eqref{eq:15} can be rewritten in matrix form as:
\begin{equation}
    \mathbf{I} = \mathbf{L} \mathbf{M},
    \label{eq:discrete_model}
\end{equation}
where $\mathbf{L} \in \R^{m \times 4}$ is the general lighting matrix, and $\mathbf{M} \in \R^{4 \times n}$ stacks all the unknown $\mathbf{m}(x)$-vectors columnwise (each column $\mathbf{m}_j = \mathbf{m}(x_j)$ has thus the form given in Eq.~\eqref{eq:16}). 

As shown in~\cite{Basri2007}, a least-squares solution $(\mathbf{L}_1,\mathbf{M}_1)$ of~\eqref{eq:discrete_model} satisfying the constraint~\eqref{eq:16} can be obtained by singular value decomposition of $\mathbf{I}$. Since we know that any other $\mathbf{M}$-matrix solution of~\eqref{eq:discrete_model} differs from $\mathbf{M}_1$ according to a scaled Lorentz transform, the genuine solution $\mathbf{M^*} \in \R^{4 \times n}$ is given by
\begin{equation}
\label{eq:5.0}
    \mathbf{M^*} = \mathbf{A}\mathbf{M_1},
\end{equation}
with $\mathbf{A} \in L_s$ an unknown scaled Lorentz transformation. 

In the last section we have seen that there exists a unique $\mathbf{m}$-field which both satisfies the image formation model and is integrable. This means, that if the pictured surface is twice differentiable and non degenerate, then matrix $\mathbf{A}$ in~\eqref{eq:5.0} is unique (up to scale). In fact, we only need the last three rows of matrix $\mathbf{A}$: left-multiplying the last three rows of the initial guess $\mathbf{M}_1$ by this submatrix, we obtain a matrix of size $3 \times n$ where the norm of the $j$-th column is the albedo at the surface point conjugate to pixel $x_j$, and normalizing each column yields the surface normal at this point.

The problem thus comes down to estimating the last three rows of matrix $\mathbf{A}$. According to Proposition~\ref{prop:7} in \hyperref[sec:app2]{Appendix C}, these rows can be written in the form $(\mathbf{v}\text{ | }\mathbf{Q}) \in \R^{3 \times 4}$, where $\mathbf{v} \in \R^3$ and $\mathbf{Q} \in GL(3,\R)$. Next we show how to estimate $\mathbf{v}$ and $\mathbf{Q}$ in closed-form, using a discrete analogous of the perspective integrability constraint. 

\subsection{Closed-form solution through discrete integrability enforcement}

During the proof of Theorem~\ref{thm:3}, we showed that the integrability constraint yields the set of linear systems~\eqref{eq:integMatrix_continuous} over $\Omega$. In the discrete setting, this set of equations can be written compactly as
\begin{equation}
    \mathbf{I}_p \mathbf{w} = 0, 
    \label{eq:disc_int}
\end{equation}
where $\mathbf{w} \in \R^{18}$ contains several minors of size $2$ denoted by $\left(A^{i,j}_{k,l}\right)$, and $\mathbf{I}_p \in \R^{n \times 18}$ is a ``perspective integrability matrix'' depending only upon the known camera parameters and entries of $\mathbf{M}_1$. 

Matrix $\mathbf{I}_p$ is in general full-rank. Thus, the least-squares solution (up to scale) of~\eqref{eq:disc_int} in terms of vector $\mathbf{w}$ can be determined by singular value decomposition of $\mathbf{I}_p$: denoting by $\mathbf{I_p} = \mathbf{U} \mathbf{\Sigma} \mathbf{V}^\top$ this decomposition, the solution $\mathbf{w}$ is the last column of $\mathbf{V}$. We denote by $\left(\tilde{A}^{i,j}_{k,l}\right) = \left(\lambda\, A^{i,j}_{k,l}\right)$ its entries, where $\lambda \neq 0$ denotes the unknown scale factor. 

Now, recall that matrix $\mathbf{Q} \in \R^{3 \times 3}$ to be determined is the sub-matrix formed by the last three rows and columns of $\mathbf{A}$. It relates to the aforementioned minors according to
\begin{align}
   & \mathbf{Q}^{-1} = \dfrac{1}{\text{det}(\mathbf{Q})} \text{com}(\mathbf{Q})^\top = \dfrac{1}{\text{det}(\mathbf{Q})} \begin{pmatrix}
A^{3,3}_{4,4} & -A^{2,3}_{4,4} & A^{2,3}_{3,4} \\[4pt]
-A^{3,2}_{4,4} & A^{2,2}_{4,4} & -A^{2,2}_{3,4}  \\[4pt]
A^{3,2}_{4,3} & -A^{2,2}_{4,3} & A^{2,2}_{3,3}
\end{pmatrix},
 \end{align}
where $\text{com}(\mathbf{Q})$ is the comatrix of $\mathbf{Q}$. Thus: 
\begin{align}
   & \lambda\, \mathbf{Q}^{-1} = \dfrac{1}{\text{det}(\mathbf{Q})} \mathbf{\Delta}^{-1}, \text{~where~} \mathbf{\Delta} = \begin{pmatrix}
\tilde{A}^{3,3}_{4,4} & -\tilde{A}^{2,3}_{4,4} & \tilde{A}^{2,3}_{3,4} \\[4pt]
-\tilde{A}^{3,2}_{4,4} & \tilde{A}^{2,2}_{4,4} & -\tilde{A}^{2,2}_{3,4}  \\[4pt]
\tilde{A}^{3,2}_{4,3} & -\tilde{A}^{2,2}_{4,3} & \tilde{A}^{2,2}_{3,3}
\end{pmatrix}^{-1}. \label{eq:5.1}
\end{align}
Hence, we can determine $\mathbf{Q}$ up to scale:
\begin{equation}
    \label{eq:5.2}
    \mathbf{Q} = (\lambda\, \text{det}\mathbf{Q})\, \mathbf{\Delta}.
\end{equation}

Next, we turn our attention to the estimation of vector $\mathbf{v} \in \R^3$ (recall that this vector is formed by the first column and last three rows of $\mathbf{A}$), up to scale. To this end, we consider the last nine minors. For example, considering $\tilde{A}^{2,1}_{3,2}$:
\begin{align}
     \tilde{A}^{2,1}_{3,2}= &  \lambda\, (A_{21}A_{32} - A_{31}A_{22})\\
                         = &  \lambda\, (A_{21}Q_{21} - A_{31}Q_{11})\\
     \underset{\text{(\ref{eq:5.2})}}{=} &  \left(\lambda^2 \text{det}(\mathbf{Q})A_{21}\right)\Delta_{21} - \left(\lambda^2 \text{det}(\mathbf{Q})A_{31}\right)\Delta_{11}. \label{eq:5.3} 
\end{align}
Let $\hat{\mathbf{v}} = \begin{pmatrix}
\hat{v}_1 \\
\hat{v}_2 \\
\hat{v}_3
\end{pmatrix} = \lambda^2 \,\text{det}(\mathbf{Q}) \begin{pmatrix}
A_{21} \\
A_{31} \\
A_{41}
\end{pmatrix} = (\lambda^2 \text{det}(\mathbf{Q}))\mathbf{v}$. Eq.~(\ref{eq:5.3}) can be written as:
\begin{equation}
     \Delta_{21}\hat{v}_1 - \Delta_{11}\hat{v}_2 = \hat{A}^{2,1}_{3,2}.
\end{equation}
In the same manner, by using all the other minors which involve $A_{21},A_{31}$ or $A_{41}$, we get the following over-constrained linear system:
\begin{equation}
    \mathbf{S}\hat{\mathbf{v}} = \mathbf{b},
\end{equation}
where $\mathbf{S} \in \R^{9 \times 3}$ and $\mathbf{b} \in \R^9$. A least-squares solution for $\hat{\mathbf{v}}$ can be found using, e.g., the pseudo inverse:
\begin{equation}
    \hat{\mathbf{v}} = \mathbf{S}^\dagger\mathbf{b}.
\end{equation}
Besides,
\begin{align}
    \lambda^2 \text{det}(\mathbf{Q}) (\mathbf{v}\text{ | }\mathbf{Q})& = (\lambda^2 \text{det}(\mathbf{Q})\mathbf{v}\text{ | }\lambda^2 \text{det}(\mathbf{Q})\mathbf{Q})\\
    & = (\hat{\mathbf{v}}\text{ | }\lambda^3 \text{det}(\mathbf{Q})^2  \mathbf{\Delta}), \label{eq:5.4}
\end{align}
and applying the determinant to both sides of Eq.~(\ref{eq:5.2}) yields: 
\begin{equation}
    \lambda^3 \text{det}(\mathbf{Q})^2 = \dfrac{1}{\text{det}(\mathbf{\Delta})}.
    \label{eq:64}
\end{equation}

Plugging~\eqref{eq:64} into (\ref{eq:5.4}), we eventually obtain the following closed-form expression for $(\mathbf{v}\text{ | }\mathbf{Q})$:
\begin{equation}
     (\mathbf{v}\text{ | }\mathbf{Q})= \dfrac{1}{\lambda^2 \text{det}(\mathbf{Q})} \left(\hat{\mathbf{v}}\text{ | }\dfrac{1}{\text{det}(\mathbf{\Delta})} \mathbf{\Delta}\right).
     \label{eq:closedformVQ}
\end{equation}

Since $\lambda$ and $\text{det}(\mathbf{Q})$ in~\eqref{eq:closedformVQ} are unknown, the solution $(\mathbf{v}\text{ | }\mathbf{Q})$ is known only up to scale. As already stated, the actual value of this scale factor is not important, since it only scales all abedo values simultaneously without affecting the geometry. Let us denote by $\tilde{\mathbf{M}}_1$ the submatrix formed by the last three rows of the initial guess $\mathbf{M}_1$. Then, matrix $\tilde{\mathbf{M}}_2 = (\mathbf{v}\text{ | }\mathbf{Q}) \tilde{\mathbf{M}}_1$ is a $3 \times n$ matrix where each column corresponds to one surface normal, scaled by the albedo. The norm of each column of $\tilde{\mathbf{M}}_2$ thus provides the sought albedo (up to scale), and normalizing each column provides the sought surface normal. 

Therefore, we now have at hand a practical way to find an integrable normal field solving uncalibrated photometric stereo under general lighting and perspective projection. In the next subsection, we show on simulated data that such a solution indeed corresponds to the genuine surface, which provides an empirical evidence for the theoretical analysis conducted in the previous section. 

\subsection{Experiments}

To empirically validate the well-posedness of perspective uncalibrated photometric stereo under general lighting, we implemented the previous algorithm in Matlab, and evaluated it against 16 synthetic datasets. These datasets were created by considering four 3D-shapes (``Armadillo'', ``Bunny'', ``Joyful Yell''  and ``Thai Statue''\footnote{Joyful Yell: \url{https://www.thingiverse.com/thing:897412}; other datasets: \url{http://www-graphics.stanford.edu/data/3dscanrep}}) and four different albedo maps (``White'', ``Bars'', ``Ebsd'' and ``Voronoi''). Ground truth normals were deduced from the depth maps using \eqref{eq:normal_perspective}, approximating partial derivatives of the depth with first-order finite differences. Then, for each of the $16$ combinations of 3D-shape and albedo, $m=21$ images were simulated according to~\eqref{eq:14}, while varying the lighting coefficient, as illustrated in Figure~\ref{fig:1}. Each image is of size $1600 \times 1200$, and comes along with the ground-truth normals, reconstruction domain $\Omega$, and intrinsic camera parameters (the focal length ${f}$, and the principal point used as reference for pixel coordinates). For the evaluation, we measured the {mean angular error (in degrees)} between the estimated and the ground-truth normals.

\begin{figure}[!ht]
    \centering
    \includegraphics[width=.9\textwidth]{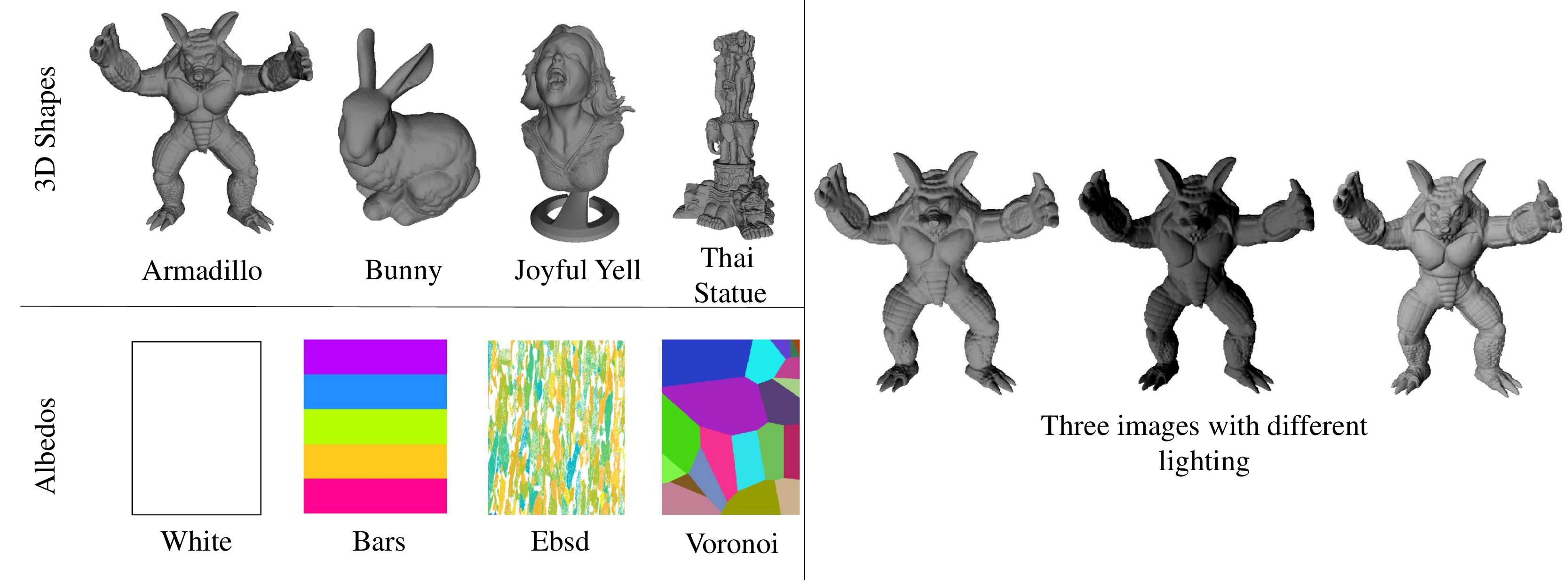}
   \caption{The four 3D-shapes and four albedo maps used to create 16 (3D-shape, albedo) datasets. For each dataset, $m=21$ images were rendered under varying first-order spherical harmonics lighting. On the right, we show three images of the (``Armadillo'', ``White'') dataset.}
   \label{fig:1}
\end{figure}

As can be seen in Table~\ref{tab:1}, the mean angular error on the estimated normals is very low (less than $10$ degrees for all datasets). This confirms that the geometry of the scene is unambiguously estimated. The images being synthesized without any additional noise or outlier to the Lambertian model (e.g., shadows or specularities), one may however be surprised that the mean angular error is non-zero. As suggested in~\cite{Papadhimitri2013}, the observed residual errors may be due to the finite differences approximation of partial derivatives arising in the perspective integrability matrix (matrix $\mathbf{I}_p$ in~\eqref{eq:disc_int}, which contains the partial derivatives of the entries of the initial $\mathbf{m}$-field, cf.~\eqref{minorsPersp}). In our implementation, we considered first-order finite differences: other choices of finite differences might reduce the error, yet we leave this as a perspective.

\begin{table}[!ht]
\centering
\begin{tabular}{l*{4}{c}}
\toprule
& \multicolumn{4}{c}{Albedo} \\
\cmidrule(lr){2-5}
3D-shape & ~~~~White~~ & ~~Bars~~ & ~~Ebsd~~ & ~~Voronoi~~~~ \\
\midrule
Armadillo & 2.01 & 1.81 & 2.13 & 2.03 \\
Bunny & 1.42 & 1.38 & 1.63 & 1.42\\
Joyful Yell & 5.13 & 5.35 & 5.46 & 5.28\\
Thai Statue & 6.33 & 6.40 & 6.46 & 7.62 \\
\bottomrule
\end{tabular}
\caption{Mean angular error (in degrees), for each (3D-shape, albedo) combination. The error remains below $10$ degrees for each dataset. This indicates that the genuine geometry is recovered, and empirically confirms the well-posedness of perspective uncalibrated photometric stereo under first-order spherical harmonics lighting.}
\label{tab:1}
\end{table}

Next, we evaluated the robustness of the proposed approach to an increasing amount of zero-mean, Gaussian noise added to the images of the (Armadillo, White) dataset. As can be seen in Table~\ref{tab:2}, the proposed method dramatically fails as soon as the noise becomes really perceptible (here, failure is observed when standard deviation $\sigma > 0.5 \%$). For comparison, we also provide the results obtained with the state-of-the-art method~\cite{Haefner2019}, which is based on heuristical shape initialization followed by regularized nonconvex refinement. The heuristical nature of the initialization induces a non-negligible bias in shape estimation, which is clearly visible on noise-free data. However, this alternative is much more robust to noise.

\begin{table}[!ht]
\centering
\begin{tabular}{l*{9}{c}}
\toprule
& \multicolumn{9}{c}{Standard deviation $\sigma$ (in percents of the maximum intensity)} \\
\cmidrule(lr){2-10}
Method & ~~~~$0.00$~~ & ~~$0.01$~~ & ~~$0.02$~~ & ~~$0.04$~~ & ~~$0.1$~~ & ~~$0.2$~~ & ~~$0.3$~~ & ~~$0.4$~~ & ~~$0.5$~~~~ \\
\midrule
\cite{Haefner2019} & 18.19 & 18.19 & 18.19 & 18.19 & 18.19 & 18.19 & 18.19 & 18.20 & \textbf{18.20} \\
Ours & \textbf{2.01} & \textbf{2.07} & \textbf{2.12} & \textbf{2.33} & \textbf{2.90} & \textbf{4.43} & \textbf{6.56} & \textbf{9.14} & 113.38  \\
\bottomrule
\end{tabular}
\caption{Mean angular error (in degrees) on the (Armadillo, White) dataset, with increasing amount of zero-mean Gaussian noise added to the input images. When noise is negligible, the proposed method largely outperforms the state-of-the-art method from~\cite{Haefner2019}. However, it should be discarded in the presence of strong noise.}
\label{tab:2}
\end{table}

This is not really surprising, since the proposed method is spectral, and the alternative one is based on evolved nonconvex optimization. In general, the former is faster (in our implementation on a recent computer, our results were obtained in less than $10$ seconds, and the alternative ones in around $30$ minutes), but the latter is more robust. Similar observations have been made in other computer vision communities, e.g. pose estimation: the $8$-point algorithm~\cite{8pointAlgorithm} is usually replaced by bundle adjustment techniques~\cite{Triggs1999} in order to handle the unavoidable noise arising in real-world data. Overall, the proposed algorithm should be considered only as a way to empirically confirm the well-posedness of the problem, yet on real-world data the existing numerical implementations should be preferred.

\section{Conclusion and perspectives}
\label{sec:6}

We have studied the well-posedness of uncalibrated photometric stereo under general illumination, represented by first-order spherical harmonics. We have established that integrability reduces the scaled Lorentz ambiguity to a global concave$/$convex ambiguity in the orthographic case, and resolves it in the perspective one. As Table~\ref{tab:3} summarizes, this generalizes previous results which were restricted to the directional lighting model. Still, open questions remain concerning further generalization of these results to even more evolved lighting models. For instance, future research on the topic may consider the case of unknown second-order spherical harmonics~\cite{Basri2007}, or that of unknown nearby point light sources~\cite{Papadhimitri2014}. Such generalizations would be of interest from a practical perspective, because the former represents natural illumination very accurately~\cite{Frolova2004}, and the latter allows using inexpensive light sources such as LEDs~\cite{Queau2018}.

\begin{table}[!ht]
\centering
\begin{tabular}{lc*{2}{c}}
\toprule
& & \multicolumn{2}{c}{Effect of imposing integrability} \\
\cmidrule(lr){3-4}
Lighting model & Underlying ambiguity & Orthographic & ~~~Perspective \\
\midrule
Directional & $9$-dof (linear)~\cite{Hayakawa1994} & $3$-dof (GBR)~\cite{Yuille1997} & Well-posed~\cite{Papadhimitri2013} \\
SH$_1$ & $6$-dof (scaled Lorentz)~\cite{Basri2007} & $\mathbf{1}$\textbf{-dof (concave$/$convex)} & \textbf{Well-posed} \\
SH$_2$ & $9$-dof (linear)~\cite{Basri2007} & ? & ?  \\
Nearby point & $4$-dof (rotation and scale)~\cite{Papadhimitri2014} & ? & ?  \\
\bottomrule
\end{tabular}
\caption{State-of-the-art of theoretical results concerning the well-posedness of uncalibrated photometric stereo under different lighting models (directional, spherical harmonics of order $1$ and $2$, or nearby point sources). We indicate the number of degrees of freedoms (dof) of the underlying ambiguity, and how imposing integrability reduces this number under both orthographic and perspective projection. The bold results refer to the findings in the present paper, and the question marks to remaining open problems.}
\label{tab:3}
\end{table}

\bibliographystyle{plain}
\bibliography{biblio}

\begin{thebibliography}{10}

\bibitem{Bartal2018}
Ofer Bartal, Nati Ofir, Yaron Lipman, and Ronen Basri.
\newblock Photometric stereo by hemispherical metric embedding.
\newblock {\em Journal of Mathematical Imaging and Vision}, 60(2):148--162,
  2018.

\bibitem{Basri2003}
Ronen Basri and David~W Jacobs.
\newblock Lambertian reflectances and linear subspaces.
\newblock {\em IEEE Transactions on Pattern Analysis and Machine Intelligence},
  25(2):218--233, 2003.

\bibitem{Basri2007}
Ronen Basri, David~W Jacobs, and Ira Kemelmacher.
\newblock Photometric stereo with general, unknown lighting.
\newblock {\em International Journal of Computer Vision}, 72(3):239--257, 2007.

\bibitem{Belhumeur1999}
Peter~N Belhumeur, David~J Kriegman, and Alan~L Yuille.
\newblock The bas-relief ambiguity.
\newblock {\em International journal of computer vision}, 35(1):33--44, 1999.

\bibitem{Breuss2012}
Michael Breu{\ss}, Emiliano Cristiani, Jean-Denis Durou, Maurizio Falcone, and
  Oliver Vogel.
\newblock Perspective shape from shading: Ambiguity analysis and numerical
  approximations.
\newblock {\em SIAM Journal on Imaging Sciences}, 5(1):311--342, 2012.

\bibitem{Chen2019b}
Guanying Chen, Kai Han, Boxin Shi, Yasuyuki Matsushita, and Kwan-Yee~K Wong.
\newblock Self-calibrating deep photometric stereo networks.
\newblock In {\em Proceedings of the IEEE Conference on Computer Vision and
  Pattern Recognition (CVPR)}, pages 8739--8747, 2019.

\bibitem{Chen2019}
Lixiong {Chen}, Yinqiang {Zheng}, Boxin {Shi}, Art {Subpa-asa}, and Imari
  {Sato}.
\newblock A microfacet-based model for photometric stereo with general
  isotropic reflectance.
\newblock {\em IEEE Transactions on Pattern Analysis and Machine Intelligence},
  (in press), 2019.

\bibitem{Durou2008}
Jean-Denis Durou, Maurizio Falcone, and Manuela Sagona.
\newblock Numerical methods for shape-from-shading: A new survey with
  benchmarks.
\newblock {\em Computer Vision and Image Understanding}, 109(1):22--43, 2008.

\bibitem{Einstein1905}
Albert Einstein.
\newblock Zur elektrodynamik bewegter k{\"o}rper.
\newblock {\em Annalen der physik}, 322(10):891--921, 1905.

\bibitem{Frolova2004}
Darya Frolova, Denis Simakov, and Ronen Basri.
\newblock {Accuracy of spherical harmonic approximations for images of
  Lambertian objects under far and near lighting}.
\newblock In {\em Proceedings of the European Conference on Computer Vision
  (ECCV)}, pages 574--587, 2004.

\bibitem{Haefner2019}
Bjoern Haefner, Zhenzhang Ye, Maolin Gao, Tao Wu, Yvain Qu\'eau, and Daniel
  Cremers.
\newblock Variational uncalibrated photometric stereo under general lighting.
\newblock In {\em Proceedings of the International Conference on Computer
  Vision (ICCV)}, 2019.

\bibitem{Hayakawa1994}
Hideki Hayakawa.
\newblock {Photometric stereo under a light source with arbitrary motion}.
\newblock {\em Journal of the Optical Society of America A}, 11(11):3079--3089,
  1994.

\bibitem{Ikehata2018}
Satoshi Ikehata.
\newblock {CNN-PS}: {CNN}-based photometric stereo for general non-convex
  surfaces.
\newblock In {\em Proceedings of the European Conference on Computer Vision
  (ECCV)}, pages 3--18, 2018.

\bibitem{Jaffe2013}
Arthur Jaffe.
\newblock Lorentz transformations, rotations, and boosts, 2013.
\newblock Online notes, accessed Sept. 2019. Url:
  \url{http://home.ku.edu.tr/~amostafazadeh/phys517_518/phys517_2016f/Handouts/A_Jaffi_Lorentz_Group.pdf}.

\bibitem{Khanian2018}
Maryam Khanian, Ali~Sharifi Boroujerdi, and Michael Breu{\ss}.
\newblock Photometric stereo for strong specular highlights.
\newblock {\em Computational Visual Media}, 4(1):83--102, 2018.

\bibitem{Kozera2018}
Ryszard Kozera and Alexander Prokopenya.
\newblock Second-order algebraic surfaces and two image photometric stereo.
\newblock In {\em International Conference on Computer Vision and Graphics},
  pages 234--247, 2018.

\bibitem{Li2019}
Junxuan Li, Antonio Robles-Kelly, Shaodi You, and Yasuyuki Matsushita.
\newblock Learning to minify photometric stereo.
\newblock In {\em Proceedings of the IEEE Conference on Computer Vision and
  Pattern Recognition (CVPR)}, pages 7568--7576, 2019.

\bibitem{Logothetis2016}
Fotios Logothetis, Roberto Mecca, Yvain Qu{'e}au, and Roberto Cipolla.
\newblock Near-field photometric stereo in ambient light.
\newblock In {\em Proceedings of the British Machine Vision Conference (BMVC)},
  2016.

\bibitem{8pointAlgorithm}
Hugh~C Longuet-Higgins.
\newblock A computer algorithm for reconstructing a scene from two projections.
\newblock In Martin~A. Fischler and Oscar Firschein, editors, {\em Readings in
  Computer Vision}, pages 61--62. Morgan Kaufmann, 1987.

\bibitem{Mecca2013}
Roberto Mecca and Maurizio Falcone.
\newblock Uniqueness and approximation of a photometric shape-from-shading
  model.
\newblock {\em SIAM Journal on Imaging Sciences}, 6(1):616--659, 2013.

\bibitem{Mecca2016}
Roberto Mecca, Yvain Qu{\'e}au, Fotios Logothetis, and Roberto Cipolla.
\newblock A single-lobe photometric stereo approach for heterogeneous material.
\newblock {\em SIAM Journal on Imaging Sciences}, 9(4):1858--1888, 2016.

\bibitem{Mecca2014}
Roberto Mecca, Aaron Wetzler, Alfred~M Bruckstein, and Ron Kimmel.
\newblock Near field photometric stereo with point light sources.
\newblock {\em SIAM Journal on Imaging Sciences}, 7(4):2732--2770, 2014.

\bibitem{Mo2018}
Zhipeng Mo, Boxin Shi, Feng Lu, Sai-Kit Yeung, and Yasuyuki Matsushita.
\newblock Uncalibrated photometric stereo under natural illumination.
\newblock In {\em Proceedings of the IEEE Conference on Computer Vision and
  Pattern Recognition (CVPR)}, pages 2936--2945, 2018.

\bibitem{Onn1990}
Ruth Onn and Alfred Bruckstein.
\newblock Integrability disambiguates surface recovery in two-image photometric
  stereo.
\newblock {\em International Journal of Computer Vision}, 5(1):105--113, 1990.

\bibitem{Papadhimitri2013}
Thoma Papadhimitri and Paolo Favaro.
\newblock A new perspective on uncalibrated photometric stereo.
\newblock In {\em Proceedings of the IEEE Conference on Computer Vision and
  Pattern Recognition (CVPR)}, pages 1474--1481, 2013.

\bibitem{Papadhimitri2014}
Thoma Papadhimitri and Paolo Favaro.
\newblock Uncalibrated near-light photometric stereo.
\newblock In {\em Proceedings of the British Machine Vision Conference (BMVC)},
  2014.

\bibitem{Peng2017}
Songyou Peng, Bjoern Haefner, Yvain Qu{\'e}au, and Daniel Cremers.
\newblock Depth super-resolution meets uncalibrated photometric stereo.
\newblock In {\em Proceedings of the IEEE International Conference on Computer
  Vision (ICCV) Workshops}, pages 2961--2968, 2017.

\bibitem{Poincare1905}
Henri Poincar\'e.
\newblock Sur la dynamique de l{'}\'electron.
\newblock {\em Comptes Rendus de l{'}Acad\'emie des Sciences}, 140:1504--1508,
  1905.

\bibitem{Queau2018}
Yvain Qu\'eau, Bastien Durix, Tao Wu, Daniel Cremers, Fran\c{c}ois Lauze, and
  Jean-Denis Durou.
\newblock {LED}-based photometric stereo: Modeling, calibration and numerical
  solution.
\newblock {\em Journal of Mathematical Imaging and Vision}, 60(3):313--340,
  2018.

\bibitem{Queau2018b}
Yvain Qu\'{e}au, Jean-Denis Durou, and Jean-Fran{\c{c}}ois Aujol.
\newblock Normal integration: a survey.
\newblock {\em {Journal of Mathematical Imaging and Vision}}, 60(4):576--593,
  2018.

\bibitem{Queau2017b}
Yvain Qu{\'e}au, Roberto Mecca, Jean-Denis Durou, and Xavier Descombes.
\newblock Photometric stereo with only two images: A theoretical study and
  numerical resolution.
\newblock {\em Image and Vision Computing}, 57:175--191, 2017.

\bibitem{Queau2017c}
Yvain Qu{\'e}au, Tao Wu, and Daniel Cremers.
\newblock Semi-calibrated near-light photometric stereo.
\newblock In {\em International Conference on Scale Space and Variational
  Methods in Computer Vision (SSVM)}, pages 656--668, 2017.

\bibitem{Queau2017}
Yvain Qu\'eau, Tao Wu, Fran{\c{c}}ois Lauze, Jean-Denis Durou, and Daniel
  Cremers.
\newblock A non-convex variational approach to photometric stereo under
  inaccurate lighting.
\newblock In {\em Proceedings of the IEEE Conference on Computer Vision and
  Pattern Recognition (CVPR)}, pages 99--108, 2017.

\bibitem{Radow2019}
Georg Radow, Laurent Hoeltgen, Yvain Qu{\'e}au, and Michael Breu{\ss}.
\newblock Optimisation of classic photometric stereo by non-convex variational
  minimisation.
\newblock {\em Journal of Mathematical Imaging and Vision}, 61(1):84--105,
  2019.

\bibitem{Ramamoorthi2001}
Ravi Ramamoorthi and Pat Hanrahan.
\newblock An efficient representation for irradiance environment maps.
\newblock In {\em Proceedings of the Annual Conference on Computer Graphics and
  Interactive Techniques}, pages 497--500, 2001.

\bibitem{Sengupta2018}
Soumyadip Sengupta, Hao Zhou, Walter Forkel, Ronen Basri, Tom Goldstein, and
  David Jacobs.
\newblock Solving uncalibrated photometric stereo using fewer images by jointly
  optimizing low-rank matrix completion and integrability.
\newblock {\em Journal of Mathematical Imaging and Vision}, 60(4):563--575,
  2018.

\bibitem{Shi2014}
Boxin Shi, Kenji Inose, Yasuyuki Matsushita, Ping Tan, Sai-Kit Yeung, and
  Katsushi Ikeuchi.
\newblock Photometric stereo using internet images.
\newblock In {\em Proceedings of the International Conference on 3D Vision
  (3DV)}, volume~1, pages 361--368, 2014.

\bibitem{Shi2019}
Boxin Shi, Zhe Wu, Zhipeng Mo, Dinglong Duan, Sai-Kit Yeung, and Ping Tan.
\newblock A benchmark dataset and evaluation for non-{Lambertian} and
  uncalibrated photometric stereo.
\newblock {\em IEEE Transactions on Pattern Analysis and Machine Intelligence},
  41(2):271--284, 2019.

\bibitem{Tozza2016b}
Silvia Tozza and Maurizio Falcone.
\newblock {Analysis and approximation of some shape-from-shading models for
  non-Lambertian surfaces}.
\newblock {\em Journal of Mathematical Imaging and Vision}, 55(2):153--178,
  2016.

\bibitem{Tozza2016}
Silvia Tozza, Roberto Mecca, Marti Duocastella, and Alessio Del~Bue.
\newblock Direct differential photometric stereo shape recovery of diffuse and
  specular surfaces.
\newblock {\em Journal of Mathematical Imaging and Vision}, 56(1):57--76, 2016.

\bibitem{Triggs1999}
Bill Triggs, Philip~F McLauchlan, Richard~I Hartley, and Andrew~W Fitzgibbon.
\newblock Bundle adjustment -- a modern synthesis.
\newblock In {\em International Workshop on Vision Algorithms (ICCV
  Workshops)}, pages 298--372, 1999.

\bibitem{Woodham1978}
Robert~J Woodham.
\newblock Reflectance map techniques for analyzing surface defects in metal
  castings.
\newblock {Technical Report} AITR-457, MIT, 1978.

\bibitem{Woodham1980}
Robert~J Woodham.
\newblock Photometric method for determining surface orientation from multiple
  images.
\newblock {\em Optical Engineering}, 19(1):139--144, 1980.

\bibitem{Yuille1997}
Alan Yuille and Daniel Snow.
\newblock Shape and albedo from multiple images using integrability.
\newblock In {\em Proceedings of the IEEE Conference on Computer Vision and
  Pattern Recognition (CVPR)}, pages 158--164, 1997.

\bibitem{Zhang1999}
Ruo Zhang, Ping-Sing Tsai, James~Edwin Cryer, and Mubarak Shah.
\newblock Shape-from-shading: a survey.
\newblock {\em IEEE Transactions on Pattern Analysis and Machine Intelligence},
  21(8):690--706, 1999.

\end{thebibliography}

\appendix
\section*{Appendix}

\subsection*{A) Proof of Proposition~\ref{cor:2}}
\label{sec:app3}

Proposition~\ref{cor:2} characterizes the integrability of a normal field in terms of the coefficients $m_1$, $m_2$ and $m_3$ of $\mathbf{m}:=\rho \mathbf{n}$. The following proof of this proposition is largely inspired by \cite{Papadhimitri2013}. 

\begin{proof}
According to Equations~\eqref{eq:new15} to~\eqref{eq:schwarz_persp}, integrability of the normal field under perspective projection can be written as:
\begin{align}
& \hat{p}_{v} = \hat{q}_{u}\text{~over~}\Omega, \nonumber\\
 \iff & \left (\frac{p}{f - u p - v q} \right)_v = \left(\frac{q}{f - u p - v q}\right)_u\text{~over~}\Omega, \nonumber\\
 \iff & f p_{v} - v q p_{v} + v p q_{v} - f q_{u} + u p q_{u} - u q p_{u} = 0\text{~over~}\Omega, \nonumber\\
 \iff & \begin{pmatrix}
                        p_v \\
                        q_v  \\
                        0
                        \end{pmatrix}^\top \begin{pmatrix}
                        0 \\
                        -f  \\
                        v
                        \end{pmatrix} \times \begin{pmatrix}
                        p \\
                        q  \\
                        -1
                        \end{pmatrix} + \begin{pmatrix}
                        p_u \\
                        q_u  \\
                        0
                        \end{pmatrix}^\top \begin{pmatrix}
                        -f \\
                        0  \\
                        u
                        \end{pmatrix} \times \begin{pmatrix}
                        p \\
                        q  \\
                        -1 
                        \end{pmatrix} = 0\text{~over~}\Omega, \label{integVectoriel}
\end{align}
where $\times$ denotes the cross-product.

Besides, $-\dfrac{\mathbf{m}}{m_3} = \begin{pmatrix}
p \\
q  \\
-1
\end{pmatrix}$  according to~\eqref{eq:new14}. If we denote $\mathbf{w}_1 = \left[0,-f,v\right]^\top$ and $\mathbf{w}_2 = \left[-f,0,u\right]^\top$, then~\eqref{integVectoriel} yields the following equation over $\Omega$:
\begin{align}
        & \left(\dfrac{-\mathbf{m}}{m_3}\right)_v^\top \mathbf{w}_1 \times \left(\dfrac{-\mathbf{m}}{m_3}\right) + \left(\dfrac{-\mathbf{m}}{m_3}\right)_u^\top \mathbf{w}_2 \times \left(\dfrac{-\mathbf{m}}{m_3}\right)= 0,\nonumber\\
   \iff & \left(-\dfrac{m_3\mathbf{m}_v - m_{3v}\mathbf{m}}{{m_3}^2}\right)^\top \mathbf{w}_1 \times \left(\dfrac{-\mathbf{m}}{m_3}\right) +\left(-\dfrac{m_3 \mathbf{m}_u - m_{3u}\mathbf{m}}{{m_3}^2}\right)^\top \mathbf{w}_2 \times \left(\dfrac{-\mathbf{m}}{m_3}\right) = 0.\label{eq:rzqsdq}
\end{align}
Multiplying Eq.~\eqref{eq:rzqsdq} by ${m_3}^3$:
\begin{equation}
(m_3\mathbf{m}_v - m_{3v}\mathbf{m})^\top \mathbf{w}_1 \times \mathbf{m} + (m_3\mathbf{m}_u - m_{3u}\mathbf{m})^\top \mathbf{w}_2 \times \mathbf{m} = 0 \text{\quad over~}\Omega.
\end{equation}

In addition, $(\mathbf{w}_1 \times \mathbf{m}) \perp \mathbf{m}$ and $(\mathbf{w}_2 \times \mathbf{m}) \perp \mathbf{m}$, thus the following relationship holds over $\Omega$:
\begin{align}
      &  m_3\mathbf{m}_v^\top (\mathbf{w}_1 \times \mathbf{m}) + m_3\mathbf{m}_u^\top (\mathbf{w}_2 \times \mathbf{m}) = 0, \nonumber\\
\iff & \mathbf{m}_v^\top(\mathbf{w}_1 \times \mathbf{m}) + \mathbf{m}_u^\top (\mathbf{w}_2 \times \mathbf{m}) = 0, \nonumber\\
\iff &  u(m_{1u}m_{2} - m_{1}m_{2u} ) + v(m_{1v}m_{2} - m_{1}m_{2v}) \nonumber\\
& + f(m_{1v}m_{3} - m_{1}m_{3v}) - f(m_{2u}m_{3} - m_{2}m_{3u}) = 0. 
             \label{integConstraintPersp}
\end{align}
which concludes the proof.
\end{proof}

\subsection*{B) Proof of Theorem 1}
\label{sec:app1}

Theorem~\ref{thm:1} characterizes scaled Lorentz transformations. Its proof relies on the following Propositions \ref{prop:1}, \ref{prop:2} and \ref{prop:3} from Lorentz' group theory (proofs of these propositions can be found in~\cite{Jaffe2013}).

\begin{proposition}
{ For any proper and orthochronous Lorentz transformation $\mathbf{A} \in L^{p}_{o}$, there exists a unique couple $(\mathbf{v}, \mathbf{O}) \in B(\mathbf{0},1) \times SO(3,\R)$ such that}
\begin{equation}
\mathbf{A} = \mathbf{S}(\mathbf{v})\, \mathbf{R}(\mathbf{O}) = \begin{pmatrix}
\gamma & \gamma \, \mathbf{v}^\top\mathbf{O}\\
 \\
\gamma \, \mathbf{v} &\hspace{1mm} (\mathbf{I_{3}} + \frac{\gamma^{2}}{1 + \gamma} \mathbf{v} \mathbf{v}^\top) \mathbf{O} \\ 
\\
\end{pmatrix},
\end{equation}
where $\gamma = \frac{1}{\sqrt{1-\norm{\mathbf{v}}^2}}$, and 
\begin{align}
  \mathbf{S}(\mathbf{v}) = \begin{pmatrix}
\gamma & \gamma \, \mathbf{v}^\top\\
 \\
\gamma \, \mathbf{v} & \hspace{1mm} \mathbf{I_{3}} + \frac{\gamma^{2}}{1 + \gamma} \mathbf{v} \mathbf{v}^\top   \\ 
\\
\end{pmatrix}, \qquad
& \mathbf{R}(\mathbf{O}) = \begin{pmatrix}
1 & 0 & 0 & 0\\
0 \\
0 & &  \mathbf{O} \\ 
0 \\
\end{pmatrix}.
\end{align}
\label{prop:1}
\end{proposition}

\begin{proposition}
The product of two proper/improper transformations is a proper one, and the product of a proper and improper transformations is an improper one. The same for the orthochronous property.
\label{prop:2}
\end{proposition}

\begin{proposition}
Matrix $\mathbf{T} = 
\begin{pmatrix}
-1 & 0 & 0 & 0\\
0 & 1 & 0 & 0 \\
0 & 0 & 1 & 0  \\ 
0 & 0 & 0 & 1
\end{pmatrix}$ is improper and non-orthochronous, and \\
matrix $\mathbf{P} = 
\begin{pmatrix}
1 & 0 & 0 & 0\\
0 & -1 & 0 & 0 \\
0 & 0 & -1 & 0  \\ 
0 & 0 & 0 & -1
\end{pmatrix}$ is improper and orthochronous.
\label{prop:3}
\end{proposition} 

Using these already known results, we propose the following characterization of Lorentz transformations:
\begin{proposition}
{For any Lorentz transformation $\mathbf{A} \in L$, there exists a unique couple $(\mathbf{v}, \mathbf{O}) \in B(\mathbf{0},1) \times SO(3,\R)$ such that}
\begin{equation}
\mathbf{A} =
\begin{pmatrix}
\epsilon_{1}(\mathbf{A})\,\gamma & \epsilon_{1}(\mathbf{A})\,\gamma\, \mathbf{v}^{\top}\mathbf{O}\\
 \\
\epsilon_{2}(\mathbf{A})\,\gamma\, \mathbf{v} &\hspace{2mm} \epsilon_{2}(\mathbf{A})(\mathbf{I_{3}} + \frac{\gamma^{2}}{1 + \gamma} \mathbf{v} \mathbf{v}^{\top}) \mathbf{O}  
\end{pmatrix}.
\end{equation}
\label{prop:4}
\end{proposition}
\begin{proof}
{We first assume that $\mathbf{A} \in L^{i}_{n}$. \\
According to Proposition \ref{prop:3}, $\mathbf{T} \in  L^{i}_{n}$. Thus using Proposition \ref{prop:2} : $\mathbf{T}\mathbf{A} \in L^{p}_{o}$. Therefore, according to Proposition \ref{prop:1}, there exists a unique couple $(\mathbf{v}, \mathbf{O}) \in B(\mathbf{0},1) \times SO(3,\R)$ such that $ \mathbf{T}\mathbf{A} = \mathbf{S}(\mathbf{v})\mathbf{R}(\mathbf{O})$.
Since $\mathbf{T}\mathbf{T} = \mathbf{I_4}$, this implies that $\mathbf{A} = \mathbf{T}\mathbf{S}(\mathbf{v})\mathbf{R}(\mathbf{O})$.
In addition, $\epsilon_1(\mathbf{A}) = -1$ and $\epsilon_2(\mathbf{A}) = 1$,} hence:
\begin{equation}
\mathbf{A} = \begin{pmatrix}
\epsilon_{1}(\mathbf{A})\gamma & \epsilon_{1}(\mathbf{A})\gamma \mathbf{v}^\top\mathbf{O}\\
 \\
\epsilon_{2}(\mathbf{A})\gamma \mathbf{v} &\hspace{2mm} \epsilon_{2}(\mathbf{A})(\mathbf{I_{3}} + \frac{\gamma^{2}}{1 + \gamma} \mathbf{v} \mathbf{v}^\top) \mathbf{O} 
\end{pmatrix}. 
\end{equation}

With the same reasoning, we get the result for all the other transformations.
\end{proof}

Combining Proposition~\ref{prop:4} and the definition~\eqref{eq:18} of scaled Lorentz transformations, we get Theorem~\ref{thm:1}.

\subsection*{C) Some useful results on GBR and Lorentz matrices, and Corollary~\ref{cor:1}}
\label{sec:app2}

The aim of this section is to prove Corollary~\ref{cor:1}, which was used in the proofs of Theorems~\ref{thm:2} and~\ref{thm:3}. Its proof relies on a few results on GBR and Lorentz matrices, which we provide in the following.

Let us denote by $G$ the group of GBR transformations, and by $G_s$ that of scaled GBR transformations defined by:
\begin{equation}
G_s = \lbrace s\mathbf{A} \hspace{2mm} | \hspace{2mm} s \in \R\backslash \{0\} \text{ and } \mathbf{A} \in G \rbrace.
\end{equation}
Both are subgroups of $GL(3,\R)$ under the matrix product. For all $\mathbf{B} = s\mathbf{A} \in G_s$, we call $s$ the scale part of $\mathbf{B}$, and $\mathbf{A}$ its GBR part. 

Let $\mathbf{C} \in \R^{n \times  n}$ with $n > 1$ and $C_{ij}$ its entries. We will use the following notation for a minor of size two:
\begin{equation}
C^{i,j}_{k,l} =  C_{ij}C_{kl} - C_{kj}C_{il},
\end{equation}
 where $1 \leq i < k \leq n$ and $1 \leq j < l \leq n$.
 
 Such minors allow to characterize scaled GBR matrices:
\begin{proposition}
Let $\mathbf{A} \in \R^{3 \times 3}, A_{ij}$ the entries of $\mathbf{A}$. Then, $\mathbf{A}$ is a scaled GBR transformation iff $\mathbf{A}$ is invertible and fulfills the following equations:
\begin{equation}
\label{eq:equationsGBR}
\left \{
   \begin{array}{l}
     A^{2,1}_{3,2} = A^{2,1}_{3,3} = A^{1,2}_{3,3} = A^{1,1}_{3,2} = 0,\\
     A^{2,2}_{3,3} = A^{1,1}_{3,3}.
   \end{array}
   \right .
   \end{equation}
 \label{prop:5}
\end{proposition}

\begin{proof}
See~\cite{Belhumeur1999}.
\end{proof}

\begin{proposition}
Let $\mathbf{v} \in B_{0}(1), \gamma = \frac{1}{\sqrt{1-\norm{\mathbf{v}}^2}}$, then 
 $\mathbf{C} = \mathbf{I_3} + \frac{\gamma^{2}}{1 + \gamma} \mathbf{v} \mathbf{v}^\top$ is positive definite.
\label{prop:6}
\end{proposition}

\begin{proof}
Let $\mathbf{B} = \frac{\gamma^{2}}{1 + \gamma} \mathbf{v} \mathbf{v}^\top$. We note $E_\lambda(\mathbf{B})$ the eigenspace associated to the eigenvalue $\lambda$ of $\mathbf{B}$. $\mathbf{B}$ is symmetric, thus according to the spectral theorem, all the eigenvalues of $\mathbf{B}$ are real, and $\R^3 = \bigoplus \limits_{i=1}^r E_{\lambda_i}(\mathbf{B})$ with $r \leq 3$ the number of eigenvalues, and $\lbrace \lambda_i \rbrace_{i=1..r}$ the eigenvalues of $\mathbf{B}$. Hence: $\text{dim}(\R^3) = \sum \limits_{i=1}^r \text{dim}(E_{\lambda_i}(\mathbf{B}))$. According to the rank-nullity theorem, $\text{dim}(Ker(\mathbf{B})) + rank(\mathbf{B}) = 3$, and by definition $rank(\mathbf{B}) = 1$, thus $\text{dim}(Ker(\mathbf{B})) = \text{dim}(E_0(\mathbf{B})) = 2$. We deduce that
{there exists a unique non-zero eigenvalue
$\lambda \in \R\backslash \{0\}$ such that }$\R^3 = E_0(\mathbf{B}) \bigoplus E_\lambda(\mathbf{B}) $ with $\text{dim}(E_\lambda(\mathbf{B})) = 1$.

Let $\Pi_{\vec{v}}$ the orthogonal projection onto $span \lbrace \mathbf{v} \rbrace$, and let $\mathbf{x} \in \R^3$ : $\Pi_{\mathbf{v}}(\mathbf{x}) = \frac{\mathbf{v}\mathbf{v}^\top}{\norm{\mathbf{v}}^2} \mathbf{x}$ and $\Pi_{\mathbf{v}}(\mathbf{v}) = \mathbf{v}$. We have: $\mathbf{B} = \frac{\gamma^{2}}{1 + \gamma}\norm{\mathbf{v}}^2\Pi_{\mathbf{v}}$ and $\mathbf{B}\mathbf{v} = (\frac{\gamma^{2}}{1 + \gamma}\norm{\mathbf{v}}^2)\mathbf{v} \vspace{2mm}$. Thus, $\frac{\gamma^{2}}{1 + \gamma}\norm{\mathbf{v}}^2$ is an eigenvalue of $\mathbf{B}$ and $\lambda = \frac{\gamma^{2}}{1 + \gamma}\norm{\mathbf{v}}^2$ \vspace{2mm}.  Besides, $\frac{\lambda}{\gamma-1} = \frac{\gamma^{2}}{(\gamma - 1)(\gamma + 1)}\norm{\mathbf{v}}^2 = \frac{\gamma^{2}}{\gamma^2 - 1}\norm{\mathbf{v}}^2 = \frac{1}{1 - \frac{1}{\gamma^2}}\norm{\mathbf{v}}^2 = \frac{1}{1-(1-\norm{\mathbf{v}}^2)}\norm{\mathbf{v}}^2=1$. Therefore, $\lambda = \gamma - 1 $ {and} the eigenvalues of $\mathbf{B}$ are $0$ and $(\gamma - 1)$.  Let $\alpha \in \lbrace 0,\gamma - 1 \rbrace$ and $\mathbf{u} \in E_\alpha(\mathbf{B})$. We have:
\begin{equation}
\mathbf{B}\mathbf{u}= \alpha\mathbf{u}  \iff \mathbf{u} + \mathbf{B}\mathbf{u} = \mathbf{u} + \alpha\mathbf{u}  \iff \mathbf{C} \mathbf{u} = (\alpha + 1)\mathbf{u}.   
\end{equation}
Thus, $1>0$ and $\gamma > 0$ are the eigenvalues of $\mathbf{C}$ with $E_1(\mathbf{C}) = E_0(\mathbf{B})$ and $E_\gamma(\mathbf{C}) = E_{\gamma-1}(\mathbf{B})$. Consequently, $\mathbf{C}$ is positive definite. 
\end{proof}

\begin{proposition}
Let $\mathbf{A}_s \in L_s$ a scaled Lorentz transformation. The submatrix $\mathbf{B}$ formed by the last 3 rows and 3 columns of $\mathbf{A}_s$ is invertible.
\label{prop:7}
\end{proposition}

\begin{proof}
{By definition of $\mathbf{A}_s$, there exists a unique couple
$(s,\mathbf{\tilde{A}}) \in \R\backslash \{0\} \times L$ such that $\mathbf{A}_s = s\mathbf{\tilde{A}}$. Hence, from Proposition \ref{prop:4}, there exists a unique couple $(\mathbf{v}, \mathbf{O}) \in B(\mathbf{0},1) \times SO(3,\R)$ such that $ \mathbf{B} = s \hspace{0.5mm} \epsilon_{2}(\mathbf{\tilde{A}})(\mathbf{I_{3}} + \frac{\gamma^{2}}{1 + \gamma} \mathbf{v} \mathbf{v}^\top) \mathbf{O}$.}  Since $\mathbf{O} \in SO(3,\R)$, $\text{det}(\mathbf{O}) = 1$. In addition, Proposition~\ref{prop:6} implies $\text{det}(\mathbf{I_{3}} + \frac{\gamma^{2}}{1 + \gamma} \mathbf{v} \mathbf{v}^\top) > 0$, thus $\text{det}(\mathbf{B}) \neq 0$.
\end{proof}

\begin{corollary}
Let $\mathbf{A}_s \in L_s$ a scaled Lorentz transformation. If its entries $A_{ij}$ fulfill:
\begin{equation}
\left \{
   \begin{array}{l}
     A^{3,2}_{4,3} = A^{3,2}_{4,4} = A^{2,3}_{4,4} = A^{2,2}_{4,3} = 0,\\
     A^{3,3}_{4,4} = A^{2,2}_{4,4},     
   \end{array}
   \right.
 \label{eq:34}  
 \end{equation}
then the submatrix $\mathbf{B}$ of $\mathbf{A}_s$ formed by the last 3 rows and 3 columns is a scaled GBR, i.e. {there exists a unique quadruple $(\lambda,\mu,\nu,\beta) \in \R^4$ with $\lambda \neq 0, \beta \neq 0$ such that} :
\begin{equation}
 \mathbf{A}_s=\begin{pmatrix}
A_{11} & A_{12} & A_{13} & A_{14} \\
A_{21} & \beta\lambda & 0 & -\beta\mu \\
A_{31} & 0 & \beta\lambda & -\beta\nu   \\
A_{41} & 0 & 0 & \beta \\ 
\end{pmatrix}.
\end{equation}
\label{cor:1}
\end{corollary}

\begin{proof}
According to Proposition \ref{prop:7}, $\mathbf{B}$ is invertible. Besides, $\mathbf{A}_s$ fulfill equations \eqref{eq:34} {iff} $\mathbf{B}$ fulfill equations (\ref{eq:equationsGBR}), thus according to Proposition \ref{prop:5}, $\mathbf{B} \in G_s$.
\end{proof}

\end{document}